\documentclass{article}
\usepackage{kr}

\usepackage{ifthen}

\newcommand{\version}{full}

\usepackage{times}
\usepackage{soul}
\usepackage{url}
\usepackage[hidelinks]{hyperref}
\usepackage[utf8]{inputenc}
\usepackage[small]{caption}
\usepackage{graphicx}
\usepackage{amsmath}
\usepackage{amsthm}
\usepackage{amsfonts}
\usepackage{booktabs}
\usepackage{algorithm}
\usepackage{algorithmic}
\usepackage{microtype}
\usepackage{xcolor}
\usepackage{xspace}
\usepackage{stmaryrd} \usepackage{paralist}

\ifthenelse{\equal{\version}{conf}}{
\usepackage[appendix=strip,bibliography=common]{apxproof}
}{
\usepackage[appendix=append,bibliography=common]{apxproof}
}

\theoremstyle{plain}
\newtheoremrep{theorem}{Theorem}[section]
\newtheoremrep{lemma}[theorem]{Lemma}
\newtheoremrep{proposition}[theorem]{Proposition}
\newtheoremrep{corollary}[theorem]{Corollary}
\newtheoremrep{claim}[theorem]{Claim}
\newtheoremrep{observation}[theorem]{Observation}
\newtheoremrep{remark}[theorem]{Remark}

\theoremstyle{definition}
\newtheoremrep{definition}[theorem]{Definition}
\newtheoremrep{example}[theorem]{Example}

 \ifthenelse{\equal{\version}{conf}}{
  \newcommand{\inConfVersion}[1]{#1} \newcommand{\inFullVersion}[1]{}   }{
  \ifthenelse{\equal{\version}{full}}{
    \newcommand{\inConfVersion}[1]{} \newcommand{\inFullVersion}[1]{#1} }{
\newcommand{\inConfVersion}[1]{\textcolor{blue}{In submission:} #1} \newcommand{\inFullVersion}[1]{\textcolor{blue}{In supplementary material:} #1} }
}

\newcommand{\rfnn}{RFNN\xspace} \newcommand{\rfnns}{RFNNs\xspace}

\newcommand{\concat}{\mathbin{|}}

\newcommand{\gnn}{GNN\xspace}  \newcommand{\gnns}{GNNs\xspace} 

\newcommand{\ac}{AC\xspace}  

\newcommand{\relu}{\text{ReLU}\xspace}  

\newcommand{\mml}{\ensuremath{\mu}\textsf{GML}\xspace}

\newcommand{\reals}{\ensuremath{\mathbb{R}}}

\newcommand{\nats}{\ensuremath{\mathbb{N}}}

\newcommand{\agg}{\textsc{Agg}} \newcommand{\comb}{\textsc{Comb}} 

\newcommand{\sem}[3]{\ensuremath{\llbracket #1 \rrbracket^{#2}_{#3}}}

\newcommand{\buildset}[2]{\ensuremath{\left\{ #1 \,\middle|\, #2 \right\}}}
\newcommand{\multileft}{\ensuremath{\{\!\!\{}}
\newcommand{\multiright}{\ensuremath{\}\!\!\}}}
\newcommand{\mset}[1]{\ensuremath{\multileft #1 \multiright}}

\newcommand{\mucalculus}{\ensuremath{\mu}\nobreakdash-calculus\xspace}

\newcommand{\binds}[2]{\ensuremath{#1|_#2}}
\newcommand{\vars}[1]{\ensuremath{\textit{vars}\left(#1\right)}}
\newcommand{\free}[1]{\ensuremath{\textit{free}\left(#1\right)}}

\newcommand{\sub}[1]{\ensuremath{\textit{sub}\left(#1\right)}}
\newcommand{\tsub}[1]{\ensuremath{\textit{sub}^+(#1)}}
\newcommand{\rsub}[1]{\ensuremath{\textit{sub}^*(#1)}}

\newcommand{\lfp}[1]{\ensuremath{\textit{sub}_{\mu}\left(#1\right)}}
\newcommand{\gfp}[1]{\ensuremath{\textit{sub}_{\nu}\left(#1\right)}}
\newcommand{\fp}[1]{\ensuremath{\textit{sub}_{\pi}\left(#1\right)}}
\newcommand{\tlfp}[1]{\ensuremath{\textit{sub}_{\mu}^+(#1)}}

\newcommand{\tfp}[1]{\ensuremath{\textit{sub}_{\pi}^+(#1)}}
\newcommand{\rlfp}[1]{\ensuremath{\textit{sub}_{\mu}^*(#1)}}
\newcommand{\rgfp}[1]{\ensuremath{\textit{sub}_{\nu}^*(#1)}}
\newcommand{\rfp}[1]{\ensuremath{\textit{sub}_{\pi}^*(#1)}}

\newcommand{\card}[1]{\ensuremath{|#1|}}
\newcommand{\defeq}{\ensuremath{:=}}
\newcommand{\pow}[1]{\ensuremath{\mathcal{P}(#1)}}
 \newcommand{\msetsfin}[1]{\ensuremath{\mathcal{M}(#1)}}
\DeclareMathOperator{\supp}{\textit{supp}} 

\newcommand{\out}[2]{\ensuremath{E_{#1}(#2)}}

\newcommand{\netw}{\ensuremath{\mathcal{N}}}

\newcommand{\lift}[1]{\ensuremath{{#1\!\!\uparrow}}}
\newcommand{\graphs}[1]{\ensuremath{\mathcal{G}[#1]}}
\newcommand{\bools}{\ensuremath{\mathbb{B}}}

\renewcommand{\vec}[1]{\ensuremath{\mathbf{#1}}}

\newcommand{\hlt}{\textsc{Hlt}}
\newcommand{\init}{\textsc{In}}
\newcommand{\readout}{\textsc{Out}}

\newcommand{\allprops}{\ensuremath{\mathbb{P}}}
\newcommand{\allvars}{\ensuremath{\mathbb{X}}}

\newcommand{\modexists}{\ensuremath{\Diamond}}
\newcommand{\atleast}[1]{\ensuremath{\Diamond_{#1}}}
\newcommand{\modforall}{\ensuremath{\Box}}
\newcommand{\allbut}[1]{\ensuremath{\Box_{#1}}}

\newcommand{\adorn}[2]{\ensuremath{{#1}^{(#2)}}}
\newcommand{\detailadorn}[3]{\ensuremath{{#1}^{(#2,#3)}}}

\newcommand{\fset}{A}
\newcommand{\counter}{C}
\newcommand{\conf}{\ensuremath{\kappa}}
\newcommand{\valuation}{\ensuremath{V}}
\newcommand{\result}{\ensuremath{R}}
\newcommand{\isvalid}{\ensuremath{F}}
\newcommand{\iskstable}{\ensuremath{S}}
\newcommand{\isckstable}{\ensuremath{T}}
\newcommand{\trans}[1]{\ensuremath{\vdash_{#1}}}
\newcommand{\ticks}[1]{\ensuremath{\textit{ticks}\left(#1\right)}}
\newcommand{\reset}[1]{\ensuremath{\textit{reset}\left(#1\right)}}

\newcommand{\dep}[1]{\ensuremath{\textit{dep}(#1)}}

\newcommand{\enc}{\textit{enc}}
\newcommand{\nodes}[1]{\ensuremath{N_{#1}}}

\newcommand{\lovaluation}{\ensuremath{v}}
\newcommand{\loresult}{\ensuremath{r}}
\newcommand{\loiskstable}{\ensuremath{s}}
\newcommand{\loisckstable}{\ensuremath{t}}
\newcommand{\residual}{\ensuremath{D}}

\newcommand{\partrans}[1]{\ensuremath{\vdash'_{#1}}}
\newcommand{\etrans}[1]{\ensuremath{\leadsto_{#1}}}

\usepackage[switch,mathlines]{lineno}
\usepackage{etoolbox}

\newcommand*\linenomathpatch[1]{\cspreto{#1}{\linenomath}\cspreto{#1*}{\linenomath}\csappto{end#1}{\endlinenomath}\csappto{end#1*}{\endlinenomath}}
\newcommand*\linenomathpatchAMS[1]{\cspreto{#1}{\linenomathAMS}\cspreto{#1*}{\linenomathAMS}\csappto{end#1}{\endlinenomath}\csappto{end#1*}{\endlinenomath}}

\expandafter\ifx\linenomath\linenomathWithnumbers
  \let\linenomathAMS\linenomathWithnumbers
\patchcmd\linenomathAMS{\advance\postdisplaypenalty\linenopenalty}{}{}{}
\else
  \let\linenomathAMS\linenomathNonumbers
\fi

\linenomathpatch{equation}
\linenomathpatchAMS{gather}
\linenomathpatchAMS{multline}
\linenomathpatchAMS{align}
\linenomathpatchAMS{alignat}
\linenomathpatchAMS{flalign}

\makeatletter
\patchcmd{\mmeasure@}{\measuring@true}{
  \measuring@true
  \ifnum-\linenopenaltypar>\interdisplaylinepenalty
    \advance\interdisplaylinepenalty-\linenopenalty
  \fi
  }{}{}
\makeatother

\title{Halting Recurrent \gnns and the Graded $\mu$-Calculus}

\author{Jeroen Bollen$^1$\and
 Jan Van den Bussche$^1$\and
 Stijn Vansummeren$^1$ \and
 Jonni Virtema$^{2}$\\
 \affiliations
 $^1$Data Science Institute, Hasselt University, Belgium\\
 $^2$School of Computer Science, University of Sheffield, UK.\\
 \emails
 \{first, second\}@uhasselt.be,
 j.t.virtema@sheffield.ac.uk
 }

\begin{document}

\maketitle

\begin{abstract}
  Graph Neural Networks (GNNs) are machine-learning
  models that operate on graph-structured data.
  Their expressive power is intimately related to logics that are
  invariant under graded bisimilarity.  Current proposals for
  recurrent GNNs either assume that the graph size is given to the
  model, or suffer
  from a lack of termination guarantees.  Here, we
  propose a halting mechanism for recurrent GNNs.  We prove that
  our model can express all node classifiers definable in
  graded modal $\mu$-calculus, even for the standard GNN variant that
  is oblivious to the graph size. To prove our main result, we develop
  a new approximate semantics for graded $\mu$-calculus, which we
  believe to be of independent interest. 
We leverage this new
  semantics into a new model-checking algorithm, called the
  counting algorithm, which is oblivious to the graph size.  In a final step we show that
  the counting algorithm can be implemented on a halting-classifier recurrent GNN.
\end{abstract}

\section{Introduction}
\label{sec:intro}

Graph neural networks (GNNs) represent a popular class of
machine-learning models on graphs \cite{hamilton-grl}.  A
multitude of GNN variants have been proposed
\cite{sato-gnn-survey,wuComprehensiveSurveyGraph2021,china-gnn-survey},
but in their basic form,
a GNN updates a vector of numerical features in every node of a
graph by combining the node's own feature vector with the sum of
those of its neighbors.  The combination is usually expressed by
a feedforward neural network with ReLU activation functions
\cite{goodfellow-book}.  The parameters of this network are
typically learned, but in this paper we are concerned with the
intrinsic expressiveness of the GNN model, and not on how GNNs
can be concretely obtained from training data.

The ``message passing'' \cite{gilmer-mpnn} between neighbors in
the graph, just described, starts from an initial feature map and
can go on for a fixed or variable number of rounds (referred to
as \emph{fixed-depth} or \emph{recurrent} GNNs, respectively).
Ultimately, of course, we want to do something with the feature
vectors computed for the nodes.  We focus on the task of
\emph{node classification} where in the end a boolean-valued
classification function is applied to the feature vector of each
node.  In this way, GNNs express unary (i.e., node-selecting)
queries on graphs.  In 
graph learning, unary queries
on graphs are known as node classifiers.

It is natural to ask about the power of GNNs in expressing node
classifiers.  Interestingly, this question can be approached
through logic.  Initial results focused on the question of
\emph{distinguishability:} given two finite pointed graphs
$(G,v)$ and $(H,w)$, does there exist a GNN $\mathcal N$ such
that $\mathcal N(G,v)$ is true but $\mathcal N(H,w)$ is false?
Distinguishability by GNNs was found to be closely related to
distinguishability by color refinement.  Specifically, it is easy to
see that when $(G,v)$ and $(H,w)$ are graded bisimilar
(see, e.g., \cite{otto-graded} for a definition), which is equivalent to indistinguishability
by color refinement, then $(G,v)$ and $(H,w)$ are
indistinguishable by GNNs.  It turns out that the converse
implication holds as well \cite{grohe-logic-gnn}.

For distinguishing finite graphs, it does not matter
whether we work with fixed-depth or recurrent GNNs.  This
changes when considering \emph{uniform expressiveness;} that is, the question which
unary graph queries are expressible by GNNs?  We saw above that all
expressible graph queries are invariant under
graded bisimilarity. 
An important related logic is \emph{graded
modal logic} GML \cite{rijke-gml}, for  it can express all so-called
\emph{logical classifiers} (i.e., unary queries that are expressible in first-order logic)
that are \emph{invariant} under graded bisimulation \cite{otto-graded}.
Interestingly, it has been
shown that every GML formula is expressible by a fixed-depth GNN
\cite{barcelo-log-expr-gnn}.

What about recurrent GNNs?  A logical step is to gauge their
expressiveness through the extension of GML with recursion, i.e.,
the \emph{graded $\mu$-calculus} \mml \cite{ksv-gradedmu}.  This logic is
not only fundamental in computer-aided verification, but also
lies at the basis of expressive description logics.  Our main
result 
shows that every $\mml$
formula is expressible by a recurrent GNN\@.  Two important
remarks 
are in order here.  First, we work with
the plain-vanilla variety of GNNs: the combination function
is a feedforward network employing ReLU.
Second, our recurrent GNNs are \emph{halting}.  By this we mean
that the GNN includes a boolean-valued halting function, defined
on feature vectors.  The GNN we construct for any $\mml$ formula
$\varphi$ comes with the following
halting guarantee: on every finite graph, there exists an
iteration where the halting function is true in every node.  At
that point, the classification function can be applied at every
node and will be correct, i.e., agree with $\varphi$.
This global halting condition will be reached in a number of
iterations that is polynomial in the graph size. 

Getting such a halting guarantee is quite challenging, because a
GNN operating on a graph $G$ cannot know the precise number $N$
of nodes of $G$.  Therefore, we cannot simply iterate fixpoint
formulas $N$ times to obtain the correct result.  Extensions of
GNNs that know the graph size, e.g., by global readout layers
\cite{barcelo-log-expr-gnn}, or simply by setting the graph size
at initialisation \cite{kostylev-recgnn}, have been considered.
Our contribution is to show that global readouts are not
necessary for expressing $\mml$, if one adopts a
global halting condition instead.  This result is not only
interesting from the perspective of fundamental understanding;
simpler neural network architectures also tend to be easier to
train, although experimental follow-up work is necessary to
confirm this.

We prove our result in several steps.  We consider, for any
natural number $k$, the approximate semantics for $\mml$ that
iterates all fixpoints exactly $k$ times.  Of independent
interest, we define a notion of when the $k$-approximation
semantics is \emph{stable} on a graph $G$, and show that when a
$k$-approximation is stable, it coincides with the true semantics
of $\mml$.  The challenge is to show that a recurrent \gnn
can express the $k$-approximation semantics for increasing values
of $k$ as well as track stability of the current approximation.
We overcome this in two steps.  We first define an
algorithm, called the \emph{counting algorithm}, that
incrementally computes $k$-approximations and their stability.
The correctness of the algorithm, which is oblivious to the
graph size, is nontrivial.
In a second, equally crucial step, we show how to implement
the counting algorithm in a halting-classifier recurrent GNN.

This paper is organized as follows.
Section~\ref{sec:related} discusses related work.
Section~\ref{sec:prelim} provides preliminaries.
Section~\ref{sec:recurrent-gnns} introduces halting-classifier recurrent \gnns.
Section~\ref{sec:distributed_system} presents the translation of \mml into halting-classifier recurrent \gnns. 
Section~\ref{sec:conclusion} offers concluding remarks.

\inConfVersion{Proofs are omitted due to space constraints, but may be found in the technical report~\cite{recurrent-gnns-and-mu-calculus-techreport}.}
\inFullVersion{Full proofs of formal statements may be found in the Appendix.}

\section{Related work} \label{sec:related}

In comparison with related work, the innovative aspect of our
recurrent GNN model is the halting aspect.  We are aware of two prior works relating recurrent GNNs to
$\mu$-calculus.

Pflueger et al.~\shortcite{kostylev-recgnn} consider
a very general, abstract recurrent GNN model, where the combination
function can be an arbitrary function, not necessarily a feedforward network, and likewise an arbitrary
aggregation function can be used instead of summing of neighbors.
Such a general model is still invariant under graded
bisimulation.  Moreover, they do not guarantee
termination.  Instead, they classify a node $n$ as true if, during
the infinite sequence of iterated message passing, the
classification function applied to $n$ eventually
stabilizes to true.  There is no apparent way to test this
effectively.

In the setting of Pflueger et al., \emph{every}
node classifier invariant under graded bisimilarity (in
particular, every $\mml$ formula) is expressible
by a recurrent GNN\@. 
Indeed, the contribution of their work
lies much more in the converse direction.  Establishing a
preservation theorem regarding ``local'' monadic fixpoint logic
(LocMMFP) formulas invariant under graded bisimilarity, they
obtain that every recurrent GNN that expresses a node classifier
expressible in LocMMFP is actually expressible in $\mml$.

Ahvonen et al.~\shortcite{lutz-recgnn} investigate recurrent GNNs with
yet another acceptance condition, which is neither a global
halting condition (as ours) nor a local stabilization condition
(as in Pflueger et al.).  Their notion of GNN does not employ a
classification function, but instead includes a set of accepting
feature vectors. A node is classified as true if, during the
iteration, it obtains such an accepting feature vector.
Like acceptance by stabilization a la Pflueger et al., this semantics
offers no practical termination guarantee; we can see when a node
has ``decided to be true'', but there is no clear-cut way to know
that this will eventually happen, and hence no way  to know that a node will be
classified as false.

The acceptance semantics of Ahvonen et al.\@ is a double-edged sword. On the one hand, it allows
the expression of node classifiers such as the centre point
property \cite{kuusisto-centrepoint} that are not expressible in $\mml$.
On the other hand, by the lack of a global halting condition,
nested or alternating fixpoints cannot be expressed; the most we
can get using accepting feature vectors are countably infinite
disjunctions of GML formulas (denoted by $\omega$GML).  Using an
instantiation of their model over the real numbers, with
arbitrary aggregation and combination functions,
Ahvonen et al.~show that recurrent GNNs capture all of
$\omega$GML\@.  They also propose a realistic instantiation where
every GNN comes with its own finite floating-point domain,
aggregation is sum, and combination functions are truncated-ReLU
layers.  These floating-point recurrent GNNs are shown to
coincide with GMSC, a logic that iterates GML formulas. In a
technical report \cite{lutz-recgnn-arxiv}, they show that this continues to hold when using acceptance by
stabilization a la Pflueger et al. Interestingly, they show 
the real and the float instantiations to express exactly the
same node classifiers in MSO.

In parallel with our own work, Rosenbluth and Grohe
\shortcite{rosenbloed-recgnn} have studied the computational completeness of
recurrent GNNs.  In contrast to ours, their model provides the graph size as
input to the GNN, and operates on graphs with bitstring features.  Their focus
then is on showing that their model is Turing-complete for feature transformers
invariant under the colorings produced by color refinement.

 \section{Preliminaries}
\label{sec:prelim}

\noindent\textbf{Notation. }
We denote by $\nats$ and $\reals$ the sets of natural numbers and real numbers, respectively. We denote by $\bools = \{0,1\}$ the set of booleans, where we identify false with $0$ and true with $1$.
We denote the cardinality of a set $X$ by $\card{X}$, the powerset of $X$ by $\pow{X}$ and  by $\msetsfin{X}$ the set of finite multisets over $X$, i.e., the set of functions $M: X \to \nats$ whose \emph{support} $\supp(M) \defeq \{x \in X \mid M(x) > 0\}$  is finite. Intuitively, $M(x)=n$ means $M$ contains $n$ copies of $x$.
We use doubly curly braces $\mset{\ldots}$ to denote multiset comprehension. 

\vspace{1mm}
\noindent\textbf{Graphs. }
We work with finite, node-labeled, directed graphs. Given a set $X$, an \emph{$X$-labeled graph} is a triple $G = (N, E, g)$ where $N$ is a finite set of nodes; $E \subseteq N \times N$ is the edge relation, and $g\colon N \to X$ is the \emph{node labeling function}. When $X = \reals^d$ for some $d \in \nats$, we call $g$ a \emph{feature map}. To ease notation, we write $\nodes{G}$ for the node set of $G$;  $n \in G$ to indicate that $n \in \nodes{G}$;
$G(n)$  instead of $g(n)$ to denote the label of node $n$ in $G$; and $\out{G}{n}$ for the set $\{m \in N \mid (n,m) \in E\}$ of all out-neighbours of $n$ in $G$.
Our results easily extend to graphs that also admit labels on edges as well as nodes.

\vspace{1mm}
\noindent\textbf{Label transformers and node classifiers. }
We write $\graphs{X}$ for the set of all $X$-labeled graphs. A \emph{label transformer} is a function $f\colon \graphs{X} \to \graphs{Y}$ such that  $G$ and $f(G)$ have the same nodes and edges for all $G \in \graphs{X}$, but may differ in the node labels.  If $h\colon X \to Y$ is a function, then we write $\lift{h}\colon \graphs{X} \to \graphs{Y}$ for the label transformer that replaces the label of each node $n$ in the input graph $G$ by $h(n)$. We say that $\lift{h}$ is obtained by \emph{lifting } $h$. A \emph{node classifier} (on $X$) is a label transformer $f\colon \graphs{X} \to \graphs{\bools}$. 

\vspace{1mm}
\noindent\textbf{Graded modal $\mu$-calculus. }
We are interested in comparing the expressive power of recurrent \gnns to that of
the \emph{graded modal $\mu$-calculus} \mml. To define \mml, we assume given a finite set of
\emph{proposition symbols} $\allprops=\{p, q, \ldots\}$ and a countable set $\allvars = \{X, Y, \dots\}$ of \emph{variables}, disjoint with $\allprops$.

The formulae of  $\mml$ have  the following syntax, where $p \in \allprops$, $X \in \allvars$, and $k \in \nats$ with $k \geq 1$. 
\begin{multline*}
  \varphi ::= p \mid \neg p \mid X \mid \varphi \land \varphi \mid \varphi \lor \varphi \\  \mid \atleast{k}\, \varphi \mid \allbut{k}\, \varphi \mid \mu X. \varphi \mid \nu X. \varphi.
\end{multline*} 
Without loss of generality, we hence adopt formulae in negation normal form,
where the negation operator \(\neg\) can only be applied to proposition
symbols. The modal operator $\atleast{k} \varphi$ expresses that there are at
least $k$ neighbours where $\varphi$ holds while $\allbut{k} \varphi$ expresses
that there are strictly less than $k$ neighbours where $\varphi$ does not hold.
We also write $\modexists \varphi$ and $\modforall \varphi$ for
$\atleast{1} \varphi$ and $\allbut{1} \varphi$.  The fixed-point operators
\(\mu X. \varphi\) and \(\nu X. \varphi\) are respectively used to define least
and greatest fixed points, allowing for the expression of recursive properties.

Semantically, \mml formulae operate on $\pow{\allprops}$-labeled graphs. 
A \emph{valuation} on such a graph \(G\) is a function  \(V\) assigning a subset of vertices in \(G\) to each variable. The valuation \(V[X \mapsto S]\) is defined as the function that is equal to \(V\) on all variables except \(X\), which is mapped to \(S\). Given a $\pow{\allprops}$-labeled graph \(G\) and a valuation \(V\), a \mucalculus formula  $\varphi$ evaluates to a set \(\sem{\varphi}{G}{V}\) of nodes in $G$, inductively defined as follows.
\allowdisplaybreaks
\begin{align*}
  \sem{p}{G}{V} & \defeq \buildset{n \in G}{p \in G(n)} \\
  \sem{\neg p}{G}{V} & \defeq \buildset{n \in G}{p \notin G(n)} \\
  \sem{X}{G}{V} & \defeq V(X) \\
  \sem{\varphi \land \psi}{G}{V} & \defeq \sem{\varphi}{G}{V} \cap \sem{\psi}{G}{V} \\
  \sem{\varphi \lor \psi}{G}{V} & \defeq \sem{\varphi}{G}{V} \cup \sem{\psi}{G}{V} \\
  \sem{\atleast{k}{\varphi}}{G}{V} & \defeq \buildset{n \in G}{\card{ \out{G}{n} \cap \sem{\varphi}{G}{V}} \geq k } \\
  \sem{\allbut{k}{\varphi}}{G}{V} & \defeq \buildset{n \in G}{\card{ \out{G}{n} \setminus \sem{\varphi}{G}{V}} < k } \\
  \sem{\mu X. \varphi}{G}{V} & \defeq \bigcap \buildset{S \subseteq \nodes{G}}{\sem{\varphi}{G}{V[X \mapsto S]} \subseteq S} \\
  \sem{\nu X. \varphi}{G}{V} & \defeq \bigcup \buildset{S \subseteq \nodes{G}}{S \subseteq \sem{\varphi}{G}{V[X \mapsto S]}}
\end{align*}
The notions of free and bound variables are defined as usual: $\mu X. \varphi$
binds $X$ in $\varphi$ and similarly for $\nu X. \varphi$. We use
$\free{\varphi}$, and $\vars{\varphi}$ to denote the sets of free and all
variables occurring in $\varphi$, respectively. A formula is \emph{well-named}
if its free variables are distinct from its bound variables, and every variable
is bound at most once in the formula. Throughout the paper we  assume
that all considered formulae are well-named. This is without loss of generality since ill-named
formulae may always be made well-named by suitably renaming bound variables if
necessary.

A \emph{sentence} is a formula without free variables.
Note that to evaluate a sentence, the valuation \(V\) is actually not needed, as the semantics depends only the valuation of the free variables. We can then write \(\sem{\varphi}{G}{}\) to denote \(\sem{\varphi}{G}{V}\) for any valuation \(V\).

A node classifier $C$ on $\pow{\allprops}$-graphs is \emph{definable} in $\mu$-calculus if there exists a $\mu$-calculus sentence $\varphi$ such that  $\sem{\varphi}{G}{} = \{ n \in G \mid C(G)(n) = 1\}$ for all $\pow{\allprops}$-graphs $G$. We also say that $\varphi$ defines $C$ in this case.

 \section{Halting Recurrent Graph Neural Networks}
\label{sec:recurrent-gnns}

In this section we introduce our  recurrent \gnn model and show  invariance under total surjective graded bisimulations.

An \emph{aggregate-combine (AC) layer} ~\cite{barcelo-log-expr-gnn,grohe-logic-gnn,geertsExpressivePowerMessagePassing2022}
of input dimension $p$ and output dimension $q$ is a pair $L = (\agg, \comb)$ of functions  where $\agg\colon \msetsfin{\reals^p} \to \reals^h$ with $h \in \nats$ is an \emph{aggregation function}   and $\comb\colon \reals^p \times \reals^h \to \reals^q$ is a \emph{combination function}. Semantically, such a layer is a label transformer: when executed on $\reals^p$-labeled graph $G = (N,E, g)$, it returns the $\reals^q$-labeled graph $G' = (N, E, g')$ with $g'$ defined by
\[ g'\colon n \mapsto \comb \big( g(n), \agg\,\mset{g(m) \mid m \in \out{G}{n} }  \big). \]
We abuse notation, and indicate by $L$ both the \ac layer and the label transformer that it defines.

\begin{definition}[Recurrent \gnn]
  \label{def:recurrent-gnn}
  Given a finite set $X$, a \emph{halting-classifier-based recurrent \gnn over $X$ of dimension $d$} is a tuple $\netw = (\init, L, \hlt, \readout)$ where $\init\colon X \to \reals^d$ is the \emph{initialisation function}; $L$ is an \ac layer with input and output dimension $d$; $\hlt\colon \reals^d \to \bools$ is the \emph{halting function}; and $\readout\colon \reals^d \to \bools$ is the \emph{readout function}.
\end{definition}
For parsimony  we simply say ``recurrent \gnn''  instead of ``halting-classifier-based recurrent \gnn'' in what follows.

We next define the semantics of recurrent \gnns.
A \emph{run} of $\netw$ over $X$-labeled graph $G$ is a finite sequence  $H_0, H_1, \dots, H_k$ of $\reals^d$-labeled graphs such that  $H_0 = \lift{\init}(G)$ and $H_{i+1} = L(H_i)$ for every $1 \leq i < k$. A run is \emph{complete} if every node in $\lift{\hlt}(H_k)$ is labeled $1$,  and $k$ is the first index for which this condition holds.
The \emph{output} of a complete run is the $\bools$-labeled graph $\lift{\readout}(H_k)$.

\begin{definition}[Halting recurrent \gnn]
  \label{def:halting-recurrent-gnn}
  A recurrent \gnn $\netw$ over $X$ is \emph{halting} if there exists a complete run for every $X$-labeled input graph. If $\netw$ is halting, then we write $\netw(G)$ for the output of the complete run of $\netw$ on $G$.
\end{definition}

Halting recurrent \gnns hence define node classifiers whereas arbitrary recurrent \gnns may define only partial node classifiers, as some of their computations (i.e., runs) may never terminate. For completeness sake, we note that the problem of determining whether a given recurrent \gnn is halting, is undecidable.  However, the GNNs that we construct in this paper are always halting.

\begin{propositionrep}[Undecidability of Halting]
  \label{prop:halting-undecidable}
  The problem of determining whether a given recurrent \gnn is halting  is undecidable.
\end{propositionrep}
  \begin{proof}
    We show that for any 3-counter machine, one can construct a simple recurrent \gnn that halts if and only if the machine does. Whether a 3-counter machine halts is undecidable, hence the problem of determining whether a simple recurrent \gnn halts on all inputs is also undecidable \cite{minsky1967}.

    A \emph{3-counter machine} \(M\) is a sequence of instructions \(I_0, \dots, I_{n-1}\) where each instruction is either an increment instruction \(\textit{inc}(r)\) for a register \(r \in \{a,b,c\}\), or a conditional jump instruction \(\text{jdz}(r, j)\) for a register \(r \in \{a,b,c\}\) and a target instruction index \(j \in \{0, \dots, n-1\}\). A configuration of the machine is a tuple \((i, v_a, v_b, v_c) \in \nats^4\) where \(i\) is the index of the current instruction and \(v_a, v_b, v_c\) are the current values of the registers. The machine halts if the instruction pointer \(i\) goes beyond \(n-1\).

    The semantics of the instructions are as follows:
    \begin{itemize}
      \item If instruction \(I_i\) is \(\textit{inc}(r)\), the next configuration is \((i+1, v'_a, v'_b, v'_c)\) where \(v'_r = v_r+1\) and the other register values remain unchanged.
      \item If instruction \(I_i\) is \(\text{jdz}(r, j)\), the next configuration is \((i+1, v_a, v_b, v_c)\) if \(v_r=0\), and if \(v_r > 0\), the next configuration is \((j, v'_a, v'_b, v'_c)\) where \(v'_r = v_r-1\) and other register values remain unchanged.
    \end{itemize}

    Given a 3-counter machine $M$ with $n$ instructions, we construct a simple recurrent \gnn $\netw$ that simulates it. The simulation operates independently on each vertex in the graph, and the feature vector of each vertex will encode the current configuration of the machine.

    Let \(m\) be the number of \(\text{jdz}\) instructions in \(M\). A configuration \((i, v_a, v_b, v_c)\) is encoded by a vector \(\vec{x} \in \reals^d\) with its components corresponding to the instruction pointer \(i\), the three counters \(a, b, c\), and for each \(\text{jdz}\) instruction in \(M\), a boolean flag used to handle the jump. As such, \(\vec{x}\) is a feature vector of dimension \(d = 4+m\). We number the \(\text{jdz}\) instructions as \(0, 1, \dots, m-1\) in the order they appear in the program. For the \(\ell\)-th \(\text{jdz}\) instruction, we use flag \(t_\ell\). Only one of these flags can be set to 1 at a time, indicating which jump we are currently processing, if any.

    Initially, all components of the vector are zero, corresponding to the initial machine configuration \((0,0,0,0)\).

    The \ac layer \(L\) of \(\netw\) updates the feature vector to simulate one step of the machine. We do not care to aggregate any incoming messages. The combination function is an \rfnn that takes the current feature vector \(\vec{x}\) and computes the next feature vector \(\vec{x}'\).

    The key difficulty is implementing the conditional jump of \(\text{jdz}(r,j)\), which requires changing the instruction pointer to an arbitrary value \(j\). As \rfnns cannot assign arbitrary values, we implement jumps iteratively. For each instruction index \(k\), let \(\mathrm{flag}(k)\) denote the flag index corresponding to instruction \(I_k\) if \(I_k\) is a \(\text{jdz}\) instruction, and undefined otherwise. When a jump from instruction \(k\) to \(j\) is required, the flag \(t_{\mathrm{flag}(k)}\) is set to 1. While any flag is 1, normal instruction execution is paused, and the instruction pointer \(i\) is moved towards \(j\) one step at a time. When \(i=j\), the flag is reset to 0.

    Let \(\vec{x}_{i}, \vec{x}_{a}, \vec{x}_{b}, \vec{x}_{c}, \vec{x}_{t_0}, \dots, \vec{x}_{t_{m-1}}\) denote the components of the feature vector \(\vec{x}\). Let \(\mathbf{t} = \sum_{\ell=0}^{m-1} \vec{x}_{t_\ell}\). The \rfnn implements the following logic:

    When \(\mathbf{t} = 0\), i.e. during normal execution, let \(k = \vec{x}_i\) be the current instruction index, so we execute instruction \(I_k\):
    \begin{align}
      \vec{x}'_i &=
      \begin{cases}
        \vec{x}_i+1 & \text{if } I_k \in \{\textit{inc}(a), \textit{inc}(b), \textit{inc}(c)\} \\
        \vec{x}_i+1 & \text{if } I_k = \text{jdz}(r,j) \text{ and } \vec{x}_r = 0 \\
        \vec{x}_i & \text{if } I_k = \text{jdz}(r,j) \text{ and } \vec{x}_r > 0
      \end{cases} \\
      \vec{x}'_a &=
      \begin{cases}
        \vec{x}_a+1 & \text{if } I_k = \textit{inc}(a) \\
        \vec{x}_a-1 & \text{if } I_k = \text{jdz}(a,j) \text{ and } \vec{x}_a > 0 \\
        \vec{x}_a & \text{otherwise}
      \end{cases} \\
      \vec{x}'_b &=
      \begin{cases}
        \vec{x}_b+1 & \text{if } I_k = \textit{inc}(b) \\
        \vec{x}_b-1 & \text{if } I_k = \text{jdz}(b,j) \text{ and } \vec{x}_b > 0 \\
        \vec{x}_b & \text{otherwise}
      \end{cases} \\
      \vec{x}'_c &=
      \begin{cases}
        \vec{x}_c+1 & \text{if } I_k = \textit{inc}(c) \\
        \vec{x}_c-1 & \text{if } I_k = \text{jdz}(c,j) \text{ and } \vec{x}_c > 0 \\
        \vec{x}_c & \text{otherwise}
      \end{cases} \\
      \vec{x}'_{t_\ell} &=
      \begin{cases}
        1 & \text{if } \ell = \mathrm{flag}(k) \text{ and } I_k = \text{jdz}(r,j) \text{ and } \vec{x}_r > 0 \\
        \vec{x}_{t_\ell} & \text{otherwise}
      \end{cases}
    \end{align}

    When \(\mathbf{t} > 0\), i.e. during the gradual execution of a jump, for the active flag \(t_\ell=1\) where the \(\ell\)-th \(\text{jdz}\) instruction is \(I_p = \text{jdz}(r,j)\) for some instruction index \(p\):
    \begin{align}
      \vec{x}'_i &=
      \begin{cases}
        \vec{x}_i+1 & \text{if } \vec{x}_i < j \\
        \vec{x}_i-1 & \text{if } \vec{x}_i > j \\
        \vec{x}_i & \text{if } \vec{x}_i = j
      \end{cases} \\
      \vec{x}'_a &= \vec{x}_a \\
      \vec{x}'_b &= \vec{x}_b \\
      \vec{x}'_c &= \vec{x}_c \\
      \vec{x}'_{t_\ell} &=
      \begin{cases}
        0 & \text{if } \vec{x}_i = j \\
        1 & \text{otherwise}
      \end{cases} \\
      \vec{x}'_{t_q} &= \vec{x}_{t_q} \quad \text{for all } q \neq \ell
    \end{align}

    All these operations (comparisons, increments, decrements) are expressible using \rfnns, as they can represent any piecewise linear function, which is sufficient for this logic.

    The halting function \(\hlt\) is defined as \(\hlt(\vec{x}) = 1\) if and only if \(\vec{x}_i \geq n\).

    Let a machine configuration be \(C_t = (i_t, v_{a,t}, v_{b,t}, v_{c,t})\) at machine step \(t\), and let the \gnn feature vector be \(\vec{x}^{(s)}\) at \gnn step \(s\). Our construction ensures that for each machine transition from \(C_t\) to \(C_{t+1}\), there exists a sequence of \gnn steps \(s_t, \dots, s_{t+1}-1\) that simulate this transition.

    An \(\textit{inc}\) instruction or a non-jumping \(\text{jdz}\) instruction corresponds to a single \gnn step. A jumping \(\text{jdz}\) instruction corresponds to one step to set the jump flag, followed by multiple steps to iteratively adjust the instruction pointer, and a final step to unset the flag.
    The \gnn halts precisely when the instruction pointer component \(\vec{x}_i\) reaches or exceeds \(n\), which corresponds exactly to the halting condition of the machine \(M\).

    Since this construction can be done for any 3-counter machine, and the halting problem for such machines is undecidable, it follows that the problem of determining whether a given recurrent \gnn halts on all inputs is also undecidable.
  \end{proof}
  The proof is based on a reduction from the (undecidable) halting problem for
  3-counter machines \cite{minsky1967}.

\vspace{1mm}
\noindent\textbf{Simple recurrent \gnns. }
A simple but practically relevant choice for the aggregation and
combination functions of an \ac layer, which is commonly used in the literature
\cite{barcelo-log-expr-gnn,lutz-recgnn,geertsExpressivePowerMessagePassing2022}, is to take
\begin{align*}
  \agg & \colon \msetsfin{\reals^p} \to \reals^p\colon M \mapsto  \sum_{\vec{x} \in \supp(M)} M(\vec{x}) \cdot \vec{x} \\
  \comb& \colon \reals^p \times \reals^p \to \reals^q\colon (\vec{x}, \vec{y}) \mapsto f(\vec{x} \concat \vec{y})
\end{align*}
where $M(\vec{x}) \cdot \vec{x}$ is  the scalar multiplication of $\vec{x}$ by its multiplicity in $M$;  $\vec{x} \concat \vec{y}$ denotes vector concatenation; and $f$ is a \emph{\relu-based feedforward neural network} (\rfnn).  That is, $f\colon \reals^{2p} \to \reals^q$ is
of the form
\[ A_\ell \circ \relu \circ A_{\ell-1} \circ \dots \circ \relu \circ A_1 \]
where $A_i\colon \reals^{p_i} \to \reals^{p_{i+1}}, 1 \leq i \leq \ell$ are affine transformations with $p_1 = 2p$ and $p_{\ell+1} = q$,
and $\relu$ is the Rectified Linear Unit, applying $\relu(x_i) = \max(0,x_i)$ to each vector element $x_i$ of its input vector $\vec{x}$.  If $L$ is of this form then we call $L$ \emph{simple}.

We say that the halting function $\hlt$ and readout function $\readout$ of a recurrent \gnn
are \emph{simple} if there is some $1 \leq i \leq d$ such that for all
$\vec{x} \in \reals^d$, $\hlt(\vec{x}) =1$ (resp. $\readout(\vec{x}) = 1$) if,
and only if, the $i$-th element of $\vec{x}$ is $> 0$.
\begin{definition}
  \label{def:simple-gnn}
  A recurrent \gnn is \emph{simple} if $L$, $\hlt$ and $\readout$ are all
  simple.
\end{definition}

\vspace{1mm}
\noindent\textbf{Invariance under graded bisimulation. }
As discussed in the Introduction and Section~\ref{sec:related},  a multitude of recurrent and non-recurrent \gnn variants  have been proposed in the literature. An important property in all of these variants, however, is that they are invariant under a notion of graded bisimulation.
Note that our halting condition is global and this needs to be reflected in the notion of bisimulation.
As a sanity check, therefore, we next establish invariance of our 
  recurrent \gnns under total surjective graded bisimulations.

\begin{definition}[Graded bisimulation]
  Let $G$ and $H$ be $X$-labeled graphs over a set of labels $X$. A relation $Z \subseteq \nodes{G} \times \nodes{H}$ is a \emph{graded bisimulation} (or \emph{g-bisimulation}) between $G$ and $H$ if  for every $(n, m) \in Z$ the following hold:
  \begin{enumerate}
    \item $G(n) = H(m)$,
    \item there is a bijection
      \(
        f\colon \out{G}{n} \to \out{H}{m}
      \)
      such that $(i,f(i))\in Z$, for every $i\in \out{G}{n}$.
  \end{enumerate}
The g-bisimulation is \emph{total} (\emph{surjective}, resp.), if the domain (range, resp.) of $Z$ is $\nodes{G}$ ($\nodes{H}$, resp.).
\end{definition}

  \begin{definition}
  A label transformer $f\colon \graphs{X} \to \graphs{Y}$ is \emph{invariant under g-bisimulation} if for every pair $G$ and $H$ of $X$-labeled graphs and every graded bisimulation $Z$ between $G$ and $H$, it is the case that $Z$ is also a graded bisimulation between $f(G)$ and $f(H)$. (Recall that $G$ and $f(G)$ have the same set of nodes, and similarly for $H$ and $f(H)$.) \end{definition}

\begin{proposition}\label{prop:GNNsarebis}
  Every node classifier definable by a halting recurrent \gnn is invariant under total surjective g-bisimulations.
\end{proposition}
\begin{proof}
  Notice that every aggregate-combine (AC) layer is a label transformer that is g-bisimulation invariant. Likewise the lifted initialisation  and lifted readout  functions yield g-bisimulation invariant label transformers. Since the composition of g-bisimulation invariant label transformers is a g-bisimulation invariant label transformer, the result follows by observing that the lifted halting function can be seen as a g-bisimulation invariant label transformer as well. Totality and surjectivity of the g-bisimulation quarantees that the readout function is enacted to runs of the same length.
\end{proof}

 \section{From $\mu$-calculus to halting-classifier recurrent \gnns}
\label{sec:distributed_system}

In this section we prove the central result of our paper.

\begin{theorem}
  \label{thm:mu-calculus-expressible-in-simple-gnn}
  Every node classifier defined by \mml sentence $\varphi$ is also
  definable by a simple halting recurrent \gnn.
\end{theorem}

Our proof is constructive, but requires us to develop several new concepts and
proceeds in multiple steps. First, in
Section~\ref{sec:approximations-and-safety} we show how to view the computation
of $\varphi$ as a sequence of \emph{approximations}
$\adorn{\varphi}{1}, \adorn{\varphi}{2}, \adorn{\varphi}{3}, \dots$. We observe
that this sequence reaches a fixpoint equaling $\varphi$ as soon as we reach an
approximation $\adorn{\varphi}{k}$ that is \emph{stable}
(Def.~\ref{def:pointwise-stable}), which is
guaranteed to happen when $k$ exceeds the graph size, but may also happen earlier
(Proposition~\ref{prop:smallest-stable-uniform-approximation-terminates-and-is-correct}). In Section~\ref{sec:incremental} we define an algorithm, called the \emph{counting algorithm}, for
computing the elements in the sequence of approximations and tracking their stability at the same time. The
counting algorithm is expressed as a transition system on
\emph{configurations}. We then show in Section~\ref{sec:implementing} that
configurations can be encoded as labeled graphs, and give a simple recurrent
\gnn that simulates the counting algorithm's transition system.
The \gnn's halting classifier tests the stability of the current encoded configuration to decide if $\varphi$ is fully computed.

We will require the following notation.  For a \mml formula $\varphi$ we
write $\sub{\varphi}$ for the set of direct subformulae of $\varphi$. For
example, if $\varphi = \neg p \vee (X \wedge \atleast{} q)$, then
$\sub{\varphi}$ consists of $\neg p$ and $X \wedge \atleast{} q$. Note that
$\sub{p} = \sub{\neg p} = \sub{X} = \emptyset$. We write $\tsub{\varphi}$ for
the set of all strict (not necessarily direct) subformulae of $\varphi$,
computed recursively. In the example above, $\tsub{\varphi}$ consists of
$\neg p$, $X$, $q$, $\atleast{} q$ and $X \wedge \atleast{} q$. We write
$\rsub{\varphi}$ for $\tsub{\varphi}\cup\{\varphi\}$.  We write $\lfp{\varphi}$ (resp. $\gfp{\varphi}$) for the subset of $\sub{\varphi}$ consisting of those
formulae that are of the form $\mu X. \psi$ (resp.  $\nu X. \psi$),
and let
$\fp{\varphi} = \lfp{\varphi} \cup \gfp{\varphi}$. The notations
$\tlfp{\varphi}$, $\rlfp{\varphi}$ etc. are defined similarly as subsets of $\tsub{\varphi}$ and $\rsub{\varphi}$.
We abuse notation, and use operators that are defined on formulae also on sets
of formulae, unioning the pointwise results. For example, we write $\tfp{\fset}$
for $\bigcup_{\varphi \in \fset} \tfp{\varphi}$.
We will use the notation $\pi X. \psi$ to refer to any fixpoint formula, least or greatest (i.e. $\pi \in \{\mu,\nu\}$).

\subsection{Approximations and stability}
\label{sec:approximations-and-safety}
\label{sec:approximations-and-stability}

\begin{toappendix}
  \subsection{Proofs for Section~\ref{sec:approximations-and-stability}}
\end{toappendix}

We define the syntax of \emph{approximation-adorned} \mml
(henceforth simply called adorned \mml) to be equal to the syntax of \mml, except that all fixpoint-operators are of the form
$\mu^i X. \psi$ or $\nu^i X.  \psi$, with $i \in \nats$ and $\psi$ itself
adorned.
\begin{multline*}
  \varphi  ::= p \mid \neg p \mid X \mid \varphi \land \varphi \mid \varphi \lor \varphi \\  \mid \atleast{k} \varphi \mid \allbut{k} \varphi \mid \mu^i X. \varphi \mid \nu^i X. \varphi.
\end{multline*}
The semantics of adorned formulae is defined similarly to that of normal formulae, except that
\begin{align*}
  \sem{\mu^i X. \varphi}{G}{V} & \defeq
  \begin{cases}
    \emptyset & \text{if } i = 0, \\
    \sem{\varphi}{G}{V[X \mapsto \sem{\mu^{i-1} X. \varphi}{G}{V}]} & \text{otherwise.}
  \end{cases}
  \\
  \sem{\nu^i X. \varphi}{G}{V} & \defeq
  \begin{cases}
    N & \text{if } i = 0, \\
    \sem{\varphi}{G}{V[X \mapsto \sem{\nu^{i-1} X. \varphi}{G}{V}]} & \text{otherwise.}
  \end{cases}
\end{align*}
Intuitively, $\mu^i X. \varphi$ computes an under-approximation of $\mu X. \varphi$, obtained by iterating $\varphi$ for $i$ iterations, while $\nu^i X. \varphi$ similarly computes an over-approximation of $\nu X. \varphi$.

For a normal \mml formula $\varphi$ and $i \in \nats$, we denote by $\adorn{\varphi}{i}$ the adorned \mml formula that is obtained by adorning every fixpoint operator of the form $\mu X$ resp $\nu X$ by $\mu^i X$ resp. $\nu^i X$. For instance:
\begin{align*}
  \varphi & = \mu Y. \left( \left(p \vee \ \atleast{} Y\right) \vee \left(\mu X. \left(q \wedge \atleast{} \left(Y \vee \atleast{} X\right)\right) \right) \right) \\
  \adorn{\varphi}{i} & = \mu^i Y. \left( \left(p \vee \ \atleast{} Y\right) \vee \left(\mu^i X. \left(q \wedge \atleast{} \left(Y \vee \atleast{} X\right)\right) \right) \right)
\end{align*}
Intuitively, $\adorn{\varphi}{i}$ approximates $\varphi$ by iterating every fixpoint for $i$ times. We also call
$\adorn{\varphi}{i}$ the \emph{$i$-th uniform approximation} of $\varphi$. Note
that, because $\varphi$ may have nested and alternating fixpoints,
$\adorn{\varphi}{i}$ itself is not necessarily an under-approximation, nor an
over-approximation of $\varphi$.

For a fixpoint formula $\varphi= \pi X. \psi$ and $i,k \in \nats$ we write $\detailadorn{\varphi}{i}{k}$ for the adorned formula $\pi^{i} X. \adorn{\psi}{k}$ that iterates the outermost fixpoint $i$
times, and all inner fixpoints $k$ times. For instance:
\begin{align*}
  \varphi & = \mu Y. \left( \left(p \vee \ \atleast{} Y\right) \vee \left(\mu X. \left(q \wedge \atleast{} \left(Y \vee \atleast{} X\right)\right) \right) \right) \\
  \detailadorn{\varphi}{i}{k} & = \mu^i Y. \left( \left(p \vee \ \atleast{} Y\right) \vee \left(\mu^k X. \left(q \wedge \atleast{} \left(Y \vee \atleast{} X\right)\right) \right) \right)
\end{align*}

\begin{definition}
  \label{def:pointwise-stable}
  Let $k \in \nats, k \geq 1$. A \mml formula  $\varphi$ is \emph{$k$-stable on input graph $G$, valuation $\valuation$, and node $n \in G$} if
  \begin{itemize}
    \item $\varphi$ is not a fixpoint formula (i.e., not of the form $\pi X. \psi$) and every direct subformula of $\varphi$ is
      $k$-stable on $(G,V,n)$. In particular $p$, $\neg p$, and $X$ are always
      $k$-stable on $(G,V,n)$ as they do not have direct subformulae.
    \item Or, $\varphi$ is a fixpoint formula $\pi X. \psi$ and
      \begin{enumerate}
        \item $n \in \sem{\detailadorn{\varphi}{k}{k}}{G}{V}$ iff $n \in  \sem{\detailadorn{\varphi}{k-1}{k}}{G}{V}$; and
        \item for every $0 \leq i < k$, $\psi$ is $k$-stable on $(G,V_i,n)$ where $V_i = V[X \mapsto \sem{\detailadorn{\varphi}{i}{k}}{G}{V}]$.
      \end{enumerate}
  \end{itemize}
  Formula $\varphi$ is \emph{$k$-stable on $(G,V)$} if it is $k$-stable on $(G,V,n)$ for every $n \in G$.
\end{definition}

\begin{toappendix}
  \label{sec:proof-approximation}

Definition~\ref{def:pointwise-stable} defines \emph{pointwise} $k$ stability,
taken at a single node $n \in G$. For the proof of
Proposition~\ref{prop:smallest-stable-uniform-approximation-terminates-and-is-correct}, it is convenient to also define \emph{global} $k$-stability.

\begin{definition}
  \label{def:global-stable}
  Let $k \in \nats, k \geq 1$. A \mml formula  $\varphi$ is \emph{globally $k$-stable on input graph $G$ and valuation $V$} if
  \begin{itemize}
    \item $\varphi$ is not a fixpoint formula and every $\psi \in \sub{\varphi}$ is
      $k$-stable on $(G,V)$. In particular $p$, $\neg p$, and $X$ are always
      $k$-stable on $(G,V)$ as they do not have strict subformulae.
    \item Or, $\varphi$ is a fixpoint formula $\pi X. \psi$ and
      \begin{enumerate}
        \item $\sem{\detailadorn{\varphi}{k}{k}}{G}{V} = \sem{\detailadorn{\varphi}{k-1}{k}}{G}{V}$; and
        \item for every $0 \leq i < k$, $\psi$ is $k$-stable on $(G,V_i)$ where $V_i = V[X \mapsto \sem{\detailadorn{\varphi}{i}{k}}{G}{V}]$.
      \end{enumerate}
  \end{itemize}
\end{definition}

\begin{lemma}
  \label{lem:global-stability-equals-pointwise-stability-at-every-node}
  $\varphi$ is globally $k$-stable on $(G,\valuation)$ if, and only if $\varphi$ is (pointwise) $k$-stable on $(G,\valuation, n)$ for every $n\in G$.
\end{lemma}
\begin{proof}
  By straightforward induction on $\varphi$.
\end{proof}

The following proposition implies that if a formula is globally $k$-stable for a given $k$, then it is also globally $l$-stable for all $l \geq k$. Moreover, the value computed by the uniform $k$-approximation equals that of the uniform $l$-approximation for $l \geq k$.

\begin{proposition}
  \label{prop:stable-preservation}
  If \mml formula $\varphi$ is globally $k$-stable on $(G,V)$ with $k \geq 1$ then:
  \begin{enumerate}
    \item[\quad (i)] $\sem{\adorn{\varphi}{k}}{G}{V} = \sem{\adorn{\varphi}{k+1}}{G}{V}$; and
    \item[\quad (ii)] $\varphi$ is globally $(k+1)$-stable on $(G,V)$.
  \end{enumerate}
\end{proposition}
\begin{proof}
  Assume $k \geq 1$ and $\varphi$ $k$-stable on $(G,V)$.
  The proof is by structural induction on $\varphi$. For the base cases where $\varphi$ is $p$, $\neg p$ or $X$ the result is immediate. For the cases where $\varphi$ is of the form $\psi \vee \psi'$, $\psi \wedge \psi'$, $\atleast{\ell}\, \psi$, or $\allbut{\ell}\, \psi$ the result follows straightforwardly from the induction hypothesis. When $\varphi = \mu X. \psi$ we reason as follows.

  For $0 \leq i < k$, let $V_i = V[X \mapsto \sem{\mu^i X. \adorn{\psi}{k}}{G}{V}]$. Because $\psi$ is globally $k$-stable on $(G,V_i)$ for each $0 \leq i < k$, we know by induction hypothesis that $\sem{\adorn{\psi}{k}}{G}{V_i} = \sem{\adorn{\psi}{k+1}}{G}{V_i}$ and that $\psi$ is globally $(k+1)$ stable w.r.t. $(G, V_i)$.

  To prove property (i), we first claim that for any $0 \leq i \leq k$ we have
  \begin{equation}
    \label{eq:stable-preservation-claim}
    \sem{\mu^i X. \adorn{\psi}{k}}{G}{V} = \sem{\mu^i X. \adorn{\psi}{k+1}}{G}{V}
  \end{equation}
  The proof of this claim is by an inner induction on $i$.
  \begin{itemize}
    \item When $i = 0$ the reasoning is trivial, since both
      $\sem{\mu^i X. \adorn{\psi}{k}}{G}{V} = \emptyset$ and
      $\sem{\mu^i X. \adorn{\psi}{k+1}}{G}{V} = \emptyset$ by definition of the
      adorned semantics.
    \item When $i > 0$, first observe that, by the inner induction hypothesis,
      \begin{align}
        \label{eq:stable-preservation-1}
        V_{i-1} & = V[X \mapsto \sem{\mu^{i-1} X. \adorn{\psi}{k}}{G}{V}] \nonumber \\
        &  = V[X \mapsto \sem{\mu^{i-1} X. \adorn{\psi}{k+1}}{G}{V}].
      \end{align}
      Therefore,
      \begin{align*}
        \sem{\mu^i X. \adorn{\psi}{k}}{G}{V} & = \sem{\adorn{\psi}{k}}{G}{V_{i-1}} \\
        & = \sem{\adorn{\psi}{k+1}}{G}{V_{i-1}} \\
        & = \sem{\adorn{\psi}{k+1}}{G}{V[X \mapsto \sem{\mu^{i-1} X. \adorn{\psi}{k+1}}{G}{V}]} \\
        & = \sem{\mu^i X. \adorn{\psi}{k+1}}{G}{V}
      \end{align*}
      The first equality is by definition of the adorned semantics; the second by the outer induction hypothesis; the third by \eqref{eq:stable-preservation-1}; and the last again by definition of approximate semantics.
  \end{itemize}
  Because $\varphi$ is globally $k$-stable, we know that $\sem{\mu^k X. \adorn{\psi}{k}}{G}{V} = \sem{\mu^{k-1} X. \adorn{\psi}{k}}{G}{V}$. Therefore, by \eqref{eq:stable-preservation-claim},  also
  \begin{equation}
    \label{eq:stable-preservation-2}
    \sem{\mu^k X. \adorn{\psi}{k+1}}{G}{V} = \sem{\mu^{k-1} X. \adorn{\psi}{k+1}}{G}{V}.
  \end{equation}
  Hence,
  \begin{align}
    \sem{\mu^{k+1} X. \adorn{\psi}{k+1}}{G}{V} &= \sem{\adorn{\psi}{k+1}}{G}{V[X\mapsto \sem{\mu^k X. \adorn{\psi}{k+1}}{G}{V}]} \nonumber \\
    & = \sem{\adorn{\psi}{k+1}}{G}{V[X\mapsto \sem{\mu^{k-1} X. \adorn{\psi}{k+1}}{G}{V}]} \nonumber \\
    & = \sem{\mu^{k} X. \adorn{\psi}{k+1}}{G}{V} \label{eq:stable-preservation-3}
  \end{align}
  We then readily obtain property (i):
  \begin{align*}
    \sem{\adorn{\varphi}{k}}{G}{V} &= \sem{\mu^k X. \adorn{\psi}{k}}{G}{V} \\
    & = \sem{\mu^k X. \adorn{\psi}{k+1}}{G}{V} \\
    & = \sem{\mu^{k+1} X. \adorn{\psi}{k+1}}{G}{V} \\
    & = \sem{\adorn{\varphi}{k+1}}{G}{V}
  \end{align*}
  The first equality is by definition of uniform $k$-approximation; the second by \eqref{eq:stable-preservation-claim}, the third by \eqref{eq:stable-preservation-3}, and the final equality by definition of $k+1$-approximation.

  It remains to prove that also property (ii) holds. We have already established \eqref{eq:stable-preservation-3} that $\sem{\mu^{k+1} X. \adorn{\psi}{k+1}}{G}{V} =   \sem{\mu^{k} X. \adorn{\psi}{k+1}}{G}{V}$. It remains to prove that for every $0 \leq j < k+1$, $\psi$ is globally $k+1$-stable on $(V,W_j)$, where $W_j = V[X \mapsto \sem{\mu^j X. \adorn{\psi}{k+1}}{G}{V}]$. To that end, first note that for $0 \leq j < k$, $W_j = V_j$ by \eqref{eq:stable-preservation-claim}. Furthermore, by \eqref{eq:stable-preservation-2}, $W_k = W_{k-1} = V_{k-1}$.  By the outer induction hypothesis we know that $\psi$ is globally $k+1$-stable w.r.t. $(G,V_i)$ for $1 \leq i < k$. As such, $\psi$ is $k+1$-stable w.r.t. $(G,W_j)$ for $1 \leq j < k+1$, as desired.

  The reasoning when $\varphi = \nu X.\psi$ is entirely analogous. \qedhere
\end{proof}

The following proposition shows that, for large enough values of $k$, every formula becomes globally $k$-stable and its approximate value equals the full fixpoint result.

\begin{proposition}
  \label{prop:stable-guaranteed-at-number-of-nodes}
  Let $G$ be a graph with node set $N$ and assume $k \in \nats$
  with $k > \card{N}$. Then, for any $\mu$-calculus formula $\varphi$ and valuation $V$ we have $\sem{\adorn{\varphi}{k}}{G}{V} = \sem{\varphi}{G}{V}$ and $\varphi$ is globally $k$-stable on $(G,V)$.
\end{proposition}
\begin{proof}
  Fix $G$, $k$, $\varphi$ and $V$ as stated. Observe that $k \geq 1$.

  The proof is by structural induction on $\varphi$. For the base cases where
  $\varphi$ is $p$, $\neg p$ or $X$ the result is immediate. For the cases where
  $\varphi$ is of the form $\psi \vee \psi'$, $\psi \wedge \psi'$,
  $\atleast{\ell}\, \psi$, or $\allbut{\ell}\, \psi$ the result follows straightforwardly
  from the induction hypothesis. When $\varphi = \mu X. \psi$ we reason as follows.

  For $0 \leq i < k$, let
  $V_i = V[X \mapsto \sem{\mu^i X. \adorn{\psi}{k}}{G}{V}]$. By induction
  hypothesis we know that
  $\sem{\adorn{\psi}{k}}{G}{V_i} = \sem{\psi}{G}{V_i}$.

  We claim that for any $0 \leq i \leq k$ we have
  \begin{equation}
    \label{eq:stable-guaranteed-claim}
    \sem{\mu^i X. \adorn{\psi}{k}}{G}{V} = \sem{\mu^i X. \psi}{G}{V}
  \end{equation}
  Here, on the right-hand side, $\sem{\mu^i X. \psi}{G}{V}$ for $\mml$ formula $\psi$ is inductively defined as follows, completely similar to the semantics of adorned formulae, but with the difference that any inner fixpoint formula in $\psi$ is evaluated non-approximately.
  \[
    \sem{\mu^i X. \psi}{G}{V}  =
    \begin{cases}
      \emptyset & \text{if } i = 0, \\
      \sem{\psi}{G}{V[X \mapsto \sem{\mu^{i-1} X. \psi}{G}{V}]} & \text{otherwise.}
    \end{cases}
  \]
  The proof of the claim is by an inner induction on $i$.
  \begin{itemize}
    \item When $i = 0$ the reasoning is trivial, since both
      $\sem{\mu^i X. \adorn{\psi}{k}}{G}{V} = \emptyset$ and
      $\sem{\mu^i X. \psi}{G}{V} = \emptyset$ by definition.
    \item When $i > 0$, first observe that, by the inner induction hypothesis,
      \begin{align}
        \label{eq:stable-guaranteed-1}
        V_{i-1} & = V[X \mapsto \sem{\mu^{i-1} X. \adorn{\psi}{k}}{G}{V}] \nonumber \\
        & = V[X \mapsto \sem{\mu^{i-1} X. \psi}{G}{V}].
      \end{align}
      Therefore,
      \begin{align*}
        \sem{\mu^i X. \adorn{\psi}{k}}{G}{V} & = \sem{\adorn{\psi}{k}}{G}{V_{i-1}} \\
        & = \sem{\psi}{G}{V_{i-1}} \\
        & = \sem{\psi}{G}{V[X \mapsto \sem{\mu^{i-1} X. \psi}{G}{V}]} \\
        & = \sem{\mu^i X. \psi}{G}{V}
      \end{align*}
      The first equality is by definition of the approximate semantics; the second by the outer induction hypothesis; the third by \eqref{eq:stable-guaranteed-1}; and the last again by definition of approximate semantics.
  \end{itemize}
  In particular, $\sem{\adorn{\varphi}{k}}{G}{V} = \sem{\mu^k X. \adorn{\psi}{k}}{G}{V} = \sem{\mu^k X. \psi}{G}{V}$.

  It remains to show that $\sem{\mu^k X. \psi}{G}{V} = \sem{\mu X. \psi}{G}{V}$, which is a standard textbook argument concerning least fixpoints. We include it here for completeness only.

  First observe that
  $\sem{\mu^k X. \psi}{G}{V} \subseteq \sem{\mu X. \psi}{G}{V}$. (This is
  a standard argument, by induction on $k$.)

  For the other direction, remark that, because $\mu$-calculus formulae are monotonic in the valuation\footnote{if $V(X) \subseteq W(X)$ for all $X$ then $\sem{\alpha}{G}{V} \subseteq \sem{\alpha}{G}{W}$ for all formulae $\alpha$.}, it holds that
  \[ \emptyset \subseteq \sem{\mu^1 X.\psi}{G}{V}  \subseteq \dots \subseteq \sem{\mu^{k-1} X.\psi}{G}{V} \subseteq \sem{\mu^k X.\psi}{G}{V}. \]

  Each set in this sequence is necessarily a subset of $N$. Since $N$ holds strictly less than $k$ elements, there must exist some index $0 \leq i < k$ such that $\sem{\mu^i X. \psi}{G}{V} = \sem{\mu^{i+1} X. \psi}{G}{V}$. Then, by definition of approximate semantics, also $\sem{\mu^j X. \psi}{G}{V} = \sem{\mu^{j+1} X. \psi}{G}{V}$ for every $j$ with $i \leq j < k$. So, in particular, $\sem{\mu^{k-1} X. \psi}{G}{V} = \sem{\mu^{k} X. \psi}{G}{V}$.

  Hence, $\sem{\psi}{G}{V[X\mapsto \sem{\mu^{k-1}X.\psi}{G}{V}]} = \sem{\mu^{k} X. \psi}{G}{V} \subseteq \sem{\mu^{k-1} X. \psi}{G}{V}$. In other words,
  $\sem{\mu^{k-1} X. \psi}{G}{V}  \in \buildset{S \subseteq N}{\sem{\psi}{G}{V[X \mapsto S]} \subseteq S}$.

  Therefore,
  \begin{align*}
    \sem{\mu X. \psi}{G}{V}  & = \bigcap \buildset{S \subseteq N}{\sem{\alpha}{G}{V[X \mapsto S]} \subseteq S} \\
    & \subseteq \sem{\mu^{k-1} X. \psi}{G}{V} = \sem{\mu^{k} X. \psi}{G}{V},
  \end{align*}
  as desired.

  It remains to prove that $\varphi$ is globally $k$-stable w.r.t. $(G,V)$. Note that, by our earlier reasoning,
  \begin{align*}
    \sem{\mu^k X. \adorn{\psi}{k}}{G}{V} & = \sem{\mu^k X. \psi}{G}{V} \\
    & = \sem{\mu^{k-1} X. \psi}{G}{V} = \sem{\mu^{k-1} X. \adorn{\psi}{k}}{G}{V}
  \end{align*}
  Furthermore, by the outer induction hypothesis, $\psi$ is globally $k$-stable w.r.t. \emph{any valuation}, hence this holds in particular w.r.t. for the $V_i$ with $0 \leq i < k$, as desired.
\end{proof}

 \end{toappendix}

The following proposition shows that the sequence
$\adorn{\varphi}{1}, \adorn{\varphi}{2}, \adorn{\varphi}{3}, \dots$ of
uniform approximations of $\varphi$ reaches a fixpoint that equals $\varphi$ as soon as
we reach an approximation $\adorn{\varphi}{k}$ that is $k$-\emph{stable}. This is guaranteed to happen when $k$ exceeds the graph size, but may also happen earlier. Hence, for any \mml formula $\varphi$, graph $G$, and valuation $V$
we may compute $\sem{\varphi}{G}{V}$ by means of the following simple algorithm:
\begin{quote}
  (*) Compute the $k$-approximation
  $\sem{\adorn{\varphi}{k}}{G}{V}$ for increasing values of $k$, and stop as soon
  as $\varphi$ becomes \emph{$k$-stable} on $(G,V)$. Then return $\sem{\adorn{\varphi}{k}}{G}{V}$.
\end{quote}

\begin{propositionrep}
  \label{prop:smallest-stable-uniform-approximation-terminates-and-is-correct}
  For all $k \geq 1$, $G$ and $V$ it holds that:
  \begin{enumerate}
  \item if $\varphi$ is $k$-stable on $(G,V)$ then
  $\sem{\adorn{\varphi}{k}}{G}{V} = \sem{\varphi}{G}{V}$; and
\item if $k > \card{\nodes{G}}$, then $\varphi$ is $k$-stable on $(G,V)$.
  \end{enumerate}
\end{propositionrep}
\begin{proof}
  Follows directly from Lemma~\ref{lem:global-stability-equals-pointwise-stability-at-every-node},  Proposition~\ref{prop:stable-preservation} and Proposition~\ref{prop:stable-guaranteed-at-number-of-nodes}
\end{proof}
\inFullVersion{The proof is in the Appendix.}

We define the following notion related to $k$-stability.
\begin{definition}
  \label{def:c-k-stability} Let $j,k \in \nats$ with $k \geq 1$. A \mml fixpoint
  formula $\varphi = \pi X. \psi$ is \emph{(j,k)-stable on input graph $G$,
  valuation $\valuation$ and node $n \in G$} if for every $0 \leq i < j$, $\psi$ is $k$-stable on
  $(G,V_i,n)$ where $V_i = V[X \mapsto \sem{\detailadorn{\varphi}{i}{k}}{G}{V}]$.
\end{definition}

Note that if $j = 0$, then $\varphi$ is vacuously $(0,k)$-stable on
$(G,V,n)$. Furthermore if $\varphi$ is $k$-stable on $(G,V,n)$, then it is also
$(k,k)$-stable on
$(G,V,n)$ but the converse does not hold since item (1) of
Definition~\ref{def:pointwise-stable} is not necessarily guaranteed. In general,
$(j,k)$-stability on fixpoint formulae is hence a weaker notion than
$k$-stability.

\subsection{The counting algorithm}
\label{sec:incremental}

\begin{toappendix}
  \subsection{Proofs for Section~\ref{sec:incremental}}
\end{toappendix}

To prove Theorem~\ref{thm:mu-calculus-expressible-in-simple-gnn} we next define an algorithm, called the \emph{counting algorithm}, that implements (*). To later  allow easy simulation  by means of a recurrent \gnns, we define the counting algorithm as a transition system operating on  \emph{configurations}.
We also take up the important task of proving correctness.
How to simulate the counting algorithm by means of
a recurrent \gnn is shown in Subsection~\ref{sec:implementing}.

Throughout this section and the next, let $\varphi$ be a fixed \mml sentence.
Our convention that \mml formulae are well-named implies that for every variable
$X \in \vars{\varphi}$ there exists exactly one fixpoint formula in
$\rfp{\varphi}$ that binds $X$. We denote this binding formula by
$\binds{\varphi}{X}$.

A \emph{counter} on  $\varphi$ is a mapping $\counter\colon \rfp{\varphi} \to \nats$. To ease notation, we write $\counter \leq k$ (resp. $\counter = k$) to indicate that $\counter(\alpha) \leq k$ (resp. $\counter(\alpha) = k$) for all $\alpha \in \rfp{\varphi}$.

Because the semantics of a formula only depends on its free variables, and
because the free variables of $\varphi$ and all of its subformulae are subsets
of $\vars{\varphi}$, we can treat valuations on a graph $G$ as finite mappings
$\valuation\colon \vars{\varphi} \to \pow{\nodes{G}}$. We refer to such mappings as $\varphi$-valuations. 

\begin{definition}
  Let $\varphi$ be a \mml sentence.
  A \emph{configuration} of $\varphi$ is a tuple $\conf = (G, k, \counter, \valuation, \result, \isvalid, \iskstable, \isckstable)$ where
$G$ is a $\pow{\allprops}$-labeled graph;
$k \in \nats$ with $k \geq 1$  is called the \emph{bound} of $\conf$; $\counter$ is a counter on  $\varphi$ with $\counter \leq k-1$;
$\valuation$ is a $\varphi$-valuation on $G$;
  $\result\colon \rsub{\varphi} \to \pow{\nodes{G}}$;
$\isvalid\subseteq \rsub{\varphi}$;
$\iskstable\colon \rsub{\varphi}\to \pow{\nodes{G}}$; and
$\isckstable\colon \rfp{\varphi} \to \pow{\nodes{G}}$.\end{definition}

Intuitively, the bound $k$ in a configuration will indicate the uniform approximation of $\adorn{\varphi}{k}$ for which we are currently computing $\sem{\adorn{\varphi}{k}}{G}{}$  while counter $\counter$ will record how far we are in this computation. Moreover, $\valuation$ will contain the valuation under which we are currently computing the results of subformulae, while  $\result$ stores subresults required to continue computation; $\isvalid$ registers for which subformulae of $\varphi$ the subresults stored in $\result$ can be considered valid; $\iskstable$ is used to track, for each subformula, on which nodes the subformula is $k$-stable; and $\isckstable$ is used to track, for each fixpoint subformula $\alpha$, for which nodes $n$ the subformula is $(\counter(\alpha),k)$ stable on $(G, \valuation,n)$.

Our algorithm does not work on arbitrary configurations, but on configurations for which our  intuitive description is coherent. We formalize this as follows.
To simplify notation, for a counter $\counter$ we write $\detailadorn{\alpha}{\counter}{k}$ for
$\detailadorn{\alpha}{\counter(\alpha)}{k}$ and we say that $\alpha \in \rfp{\varphi}$ is $(\counter,k)$-stable on $(G,\valuation,n)$ if $\alpha$ is $(\counter(\alpha), k)$-stable on $(G, \valuation, n)$.

\begin{definition}
  \label{def:coherent}
  A configuration $\conf$ is \emph{coherent} if the following three conditions hold.
  \begin{enumerate}
    \item It is \emph{sound}:
      $\valuation(X) =
      \sem{\detailadorn{\binds{\varphi}{X}}{\counter}{k}}{G}{\valuation}$ for all
      $X \in \vars{\varphi}$.
    \item It is \emph{consistent}: for all
      $\alpha \in \isvalid$,
      \begin{enumerate}
        \item $\result(\alpha) = \sem{\adorn{\alpha}{k}}{G}{\valuation}$;
        \item $\sub{\alpha} \subseteq \isvalid$;
        \item if $\alpha$ is a fixpoint formula, then $\counter(\alpha) = k-1$. And
      \end{enumerate}
    \item It \emph{tracks stability}:
      \begin{enumerate}
        \item $\iskstable(\alpha) = \{ n \in G \mid \alpha \text{ is $k$-stable on } (G,\valuation,n)\}$ for every $\alpha \in \isvalid$; and
        \item $\isckstable(\alpha) = \{ n \in G \mid
          \alpha \text{ is $(\counter,k)$-stable on } (G,\valuation,n) \}$ for every $\alpha \in \rfp{\varphi}$.
      \end{enumerate}
  \end{enumerate}
  A configuration $\conf$ is \emph{complete} if $\varphi \in \isvalid$. It is \emph{stable} if $\iskstable(\varphi) = \nodes{G}$ with $G$ the graph of $\conf$.
\end{definition}
Note that from a coherent and complete configuration we can obtain
$\sem{\adorn{\varphi}{k}}{G}{} =
\sem{\adorn{\varphi}{k}}{G}{\valuation}$ by simply reading
$\result(\varphi)$. Likewise, $k$-stability of $\varphi$ on
$(G,\valuation)$ can be obtained by checking that $\conf$ is stable.

The goal of the counting algorithm is to compute a coherent and complete configuration for $\varphi$ on $G$. Computation starts at the initial configuration w.r.t $k = 1$, defined below.
\begin{definition}
  The \emph{initial configuration} of $\varphi$ on graph $G$  w.r.t. $k \geq 1$ is
  $\conf \defeq (G, k, \counter, \valuation, \result, \isvalid, \iskstable, \isckstable)$ where $\counter = 0$;
  \begin{align*}
    \valuation(X) & =
    \begin{cases}
      \emptyset & \text{if } \binds{\varphi}{X} \in \rlfp{\varphi}\\
      N & \text{if } \binds{\varphi}{X} \in \rgfp{\varphi};
    \end{cases}
  \end{align*}
  $\result$ and $\iskstable$ map every formula to $\emptyset$; $\isvalid = \emptyset$; and $\isckstable$ maps every fixpoint formula to $\emptyset$.
\end{definition}

The initial configuration is trivially coherent.  We complete our definition of
the counting algorithm by defining three types of transitions on
configurations.

\begin{definition}
  \label{def:trans-one}
  Let $\conf = (G, k, \counter, \valuation, \result, \isvalid, \iskstable, \isckstable)$ be a configuration. A \emph{type-1} transition on $\conf$ yields the configuration
  $\conf' = (G, k, \counter, \valuation, \result', \isvalid', \iskstable', \isckstable)$ where
  {\small
    \allowdisplaybreaks
    \begin{align*}
      \result'(p) & \defeq \{ n \in G \mid p \in G(n) \} \\
      \result'(\neg p) & \defeq  \{ n \in G \mid p \not \in G(n) \} \\
      \result'(X) & \defeq V(X) \\
      \result'(\psi \wedge \psi') & \defeq \result(\psi) \cap \result(\psi') \\
      \result'(\psi \vee \psi') & \defeq \result(\psi) \cup \result(\psi') \\
      \result'(\atleast{\ell}\, \psi) & \defeq \{ n \in G \mid \card{\out{G}{n} \cap \result(\psi)} \geq \ell \} \\
      \result'(\allbut{\ell}\, \psi) & \defeq \{ n \in G \mid \card{\out{G}{n} \setminus \result(\psi)} < \ell \} \\
      \result'(\pi X. \psi) & \defeq \result(\psi) \\
      \isvalid' & \defeq \{ \alpha \in \rsub{\varphi} \mid \sub{\alpha} \subseteq \isvalid \} \\ & \qquad \setminus \{ \alpha \in \rfp{\varphi} \mid \counter(\alpha) < k - 1\} \\
      \iskstable'(\alpha) & \defeq   \nodes{G} \cap \bigcap_{\beta \in \sub{\alpha}} \iskstable(\beta) \qquad \text{if } \alpha \not \in \rfp{\varphi} \\
      \iskstable'(\pi X. \psi) & \defeq  \iskstable(\psi) \cap \isckstable(\pi X. \psi)
      \\ & \qquad \cap \{ n \in G \mid n \in \valuation(X) \text{ iff } n \in \result'(\psi) \}
\end{align*}}
  We denote this by $\conf \trans{1} \conf'$.
\end{definition}

Intuitively, if $\conf$ is coherent then for any formula $\beta \in \isvalid$, $\result$ already stores the value of $\sem{\adorn{\beta}{k}}{G}{\valuation}$. Hence, for any $\alpha$ with $\sub{\alpha} \subseteq \isvalid$ we may read these values to compute $\sem{\adorn{\alpha}{k}}{G}{\valuation}$, cf. the definition of $\result'$. For fixpoint formulae $\alpha = \pi X. \psi$, this is only correct if $\counter(\alpha) = k-1$, since we have then already evaluated the body $\psi$ under valuation with $\valuation(X) = \sem{\pi^{k-1} X.\psi}{G}{V}$, and hence
\begin{align*}
  \sem{\adorn{\alpha}{k}}{G}{\valuation}
  = \sem{\pi^k X. \adorn{\psi}{k}}{G}{\valuation}
  & = \sem{\adorn{\psi}{k}}{G}{\valuation[X\mapsto \sem{\detailadorn{\alpha}{k-1}{k}}{G}{\valuation}]}\\
  & = \sem{\adorn{\psi}{k}}{G}{\valuation} = \result(\psi)
\end{align*}
Here, the second-to-last equality is by soundness of $\conf$.
This reasoning is not correct when $\counter(\alpha) < k-1$, which is why all fixpoint formulae with $\counter(\alpha) < k-1$ are excluded from $\isvalid'$.
Note that $\isvalid \subseteq \isvalid'$ by consistency of $\conf$.

The construction of $\iskstable'$ follows the definition of $k$-stability. Assume that $\alpha \in \isvalid'$. If $\alpha$ is a non-fixpoint formula then it is $k$-stable on $(G,\valuation, n)$ when every subformula is $k$-stable on $(G,\valuation,n)$. If $\alpha$ is a fixpoint formula $\alpha = \pi X. \psi$, it is is $k$-stable on $(G,\valuation, n)$ if $\psi$ is $k$-stable on $(G,\valuation, n)$; $\alpha$ is itself $(\counter,k)$ stable on $(G,\valuation,n)$, and $n \in \sem{\adorn{\alpha}{k}}{G}{\valuation} = R'(\alpha)$ iff $n \in \sem{\detailadorn{\alpha}{k-1}{k}}{G}{\valuation} = \valuation(X)$.

\begin{toappendix}
  \begin{lemma}
  \label{lem:trans-1-extends-valid}
  If $\conf \trans{1} \conf'$ and $\conf$ is consistent, then $\isvalid \subseteq \isvalid'$.
\end{lemma}
\begin{proof}
  Let $\conf = (G, k, \counter, \valuation, \result, \isvalid, \iskstable, \isckstable)$ and
  $\conf' = (G, k, \counter, \valuation, \result', \isvalid', \iskstable', \isckstable)$.
  Assume $\alpha \in \isvalid$. By consistency of $\conf$,
  $\sub{\alpha} \subseteq \isvalid$. Moreover, if $\alpha$ is itself a fixpoint
  formula, $\counter(\alpha) = k-1$. Therefore, $\alpha \in \isvalid'$.
\end{proof}

\begin{lemma}
  \label{lem:trans-1-preservation-of-soundness-and-consistency}
  If $\conf \trans{1} \conf'$ and $\conf$ is sound and consistent then so is
  $\conf'$.
\end{lemma}
\begin{proof}
  Let $\conf = (G,k, \counter, \valuation, \result, \isvalid, \iskstable, \isckstable)$ and
  $\conf' = (G, k, \counter, \valuation, \result', \isvalid', \iskstable', \isckstable)$.
  Soundness of $\conf'$ is immediate, as $\conf'$ does not change its valuation.
  To prove that $\conf'$ is consistent, we need show that for all $\alpha \in \isvalid'$:
  \begin{enumerate}
    \item $\result'(\alpha) = \sem{\adorn{\alpha}{k}}{G}{\valuation}$;
    \item $\sub{\alpha} \subseteq \isvalid'$;
    \item if $\alpha$ is a fixpoint formula, then $\counter(\alpha) = k-1$.
  \end{enumerate}
  Items (2) and (3) are immediate by definition of $\isvalid'$. It hence remains
  to prove item (1).

  For all non-fixpoint formulae $\alpha$, this follows
  directly from the definition of $\result'$, the observation that if
  $\alpha \in \isvalid'$ then all immediate subformulae $\beta$ of $\alpha$ must
  be in $\isvalid$, and the consistency of $\conf$, implying that
  $\result(\beta) = \sem{\adorn{\beta}{\counter}}{G}{\valuation}$.

  We next illustrate the reasoning when $\alpha$ is a fixpoint formula
  $\alpha = \mu X. \psi$. The case when $\alpha = \nu X. \psi$ is similar. First
  observe that by definition of configuration, $\counter(\alpha) \leq k -1$.
  Assume that $\alpha \in \isvalid'$.  By definition,
  $\rsub{\psi} = \tsub{\alpha} \subseteq \isvalid$, and
  $\counter(\alpha) \not < k - 1$. So, $\counter(\alpha) = k-1$. Therefore,
  $\detailadorn{\alpha}{\counter}{k} = \mu^{k-1} X. \adorn{\psi}{k}$.Then
  \begin{align*}
    \sem{\adorn{\alpha}{k}}{G}{\valuation}
    & = \sem{\mu^k X. \adorn{\psi}{k}}{G}{\valuation} \\
    & = \sem{\adorn{\psi}{k}}{G}{\valuation[X \mapsto \sem{\mu^{k-1} X. \adorn{\psi}{k}}{G}{\valuation}]} \\
    & = \sem{\adorn{\psi}{k}}{G}{\valuation[X \mapsto \sem{\detailadorn{\alpha}{\counter}{k}}{G}{\valuation}]} \\
    & = \sem{\adorn{\psi}{k}}{G}{\valuation} \\
    & = \result(\psi) \\
    & = \result'(\mu X. \psi) \\
    & = \result'(\alpha)
  \end{align*}
  The first equality is by assumption; the second by semantics of adorned formulae; the third by our observation that $\detailadorn{\alpha}{\counter}{k} = \mu^{k-1} X. \psi$; the fourth by soundness of $\conf$ and the fact that $\alpha = \binds{\varphi}{X}$, which imply that $\valuation(X) = \sem{\detailadorn{\alpha}{\counter}{k}}{G}{\valuation}$; the fifth by consistency of $\conf$ and the fact that $\psi \in \isvalid$; and the second-to last by definition of $\result'$.
\end{proof}

 \end{toappendix}

\begin{toappendix}
  \begin{lemma}
  \label{lem:j-k-stability-implies-smaller-j-k-stability}
  If a formula $\alpha$ is $(j,k)$-stable on
  $(G, \valuation, n)$ and $j' \leq j$, then
  $\alpha$ is $(j',k)$-stable on $(G,\valuation, n)$.
\end{lemma}
\begin{proof}
  Trivial.
\end{proof}

In Section~\ref{sec:incremental} we defined fixpoint formula $\alpha$ to be
$(\counter,k)$-stable on $(G,\valuation,n)$ if $\alpha$ is
$(\counter(\alpha), k)$-stable on $(G, \valuation, n)$.

With extra guarantees, $(\counter,k)$-stability also implies $k$-stability:
\begin{lemma}
  \label{lem:c-k-stability-implies-k-stability-when-nested-fixpoints-ready}
  Let $\alpha \in \rfp{\varphi}$ be of the form $\pi X. \psi$ and $n \in G$. If (i)
  $n \in \sem{\adorn{\alpha}{k}}{G}{\valuation} \iff n \in
  \sem{\detailadorn{\alpha}{k-1}{k}}{G}{\valuation}$; (ii)
  $\counter(\alpha) = k-1$; (iii) $\alpha$ is $(\counter,k)$-stable on
  $(G,\valuation,n)$; (iv)
  $\valuation(X) = \sem{\detailadorn{\alpha}{\counter}{k}}{G}{\valuation}$;
  and (v) $\psi$ is $k$-stable on $(G,\valuation,n)$, then $\alpha$ is
  $k$-stable on $(G,\valuation,n)$.
\end{lemma}
\begin{proof}
  Let $k$, $\counter$, $\alpha$, $n$, and $\valuation$ be as stated and assume that
  (i)--(v) hold. To establish that $\alpha$ is $k$-table on $(G,\valuation,n)$.
  We need to show two things to establish $k$-stability of $\alpha$.
  \begin{itemize}
    \item
      $n \in \sem{\detailadorn{\alpha}{k}{k}}{G}{\valuation} \iff n \in
      \sem{\detailadorn{\alpha}{k-1}{k}}{G}{\valuation}$, which is true by
      assumption (i).
    \item for every $0 \leq i < k$, $\psi$ is $k$-stable on $(G,V_i,n)$ where
      $\valuation_i = \valuation[X \mapsto \sem{\detailadorn{\alpha}{i}{k}}{G}{\valuation}]$.
      Since $\counter(\alpha) = k-1$ by assumption (ii), assumption (iii) ensures
      that $\psi$ is $k$-stable on $(G, V_i, n)$ for $0 \leq i < k-1$.  Further,
      observe that $\valuation_{k-1} = \valuation$ by assumption (iv). Hence, by assumption (v), $\psi$ is also $k$-stable on $(G, \valuation_{k-1}, n)$, as desired. \qedhere
  \end{itemize}
\end{proof}

 \end{toappendix}

Based on this reasoning, we can formally prove:
\begin{lemmarep}
  \label{lem:trans-1-preservation-full}
  If $\conf \trans{1} \conf'$ and $\conf$ is coherent then so is
  $\conf'$.
\end{lemmarep}
\begin{proof}
  Let $\conf = (G,k, \counter, \valuation, \result, \isvalid, \iskstable, \isckstable)$ and
  $\conf' = (G, k, \counter, \valuation, \result', \isvalid', \iskstable', \isckstable)$. By Lemma~\ref{lem:trans-1-preservation-of-soundness-and-consistency}, $\conf'$ is sound and consistent.   It hence remains to prove that $\conf'$ tracks stability, for which we need to show that
  \begin{enumerate}
    \item $\iskstable'(\alpha) = \{ n \in G \mid \alpha \text{ is $k$-stable on } (G,\valuation,n)\}$ for all $\alpha \in \isvalid'$; and
    \item $\isckstable(\alpha) = \{ n \in G \mid \alpha \text{ is $(\counter,k)$-stable on } (G,\valuation,n)\}$ for all $\alpha \in \rfp{\varphi}$.
  \end{enumerate}
Item (2) follows immediately from the fact that $\conf$ tracks stability. It
  remains to show item (1).  Fix $\alpha \in \isvalid'$. Then $\sub{\alpha} \subseteq \isvalid$. We distinguish two cases.
  \begin{itemize}
    \item If $\alpha$ is a not a fixpoint formula then by definition of $\iskstable'(\alpha)$ we have
      \begin{align*}
        \isvalid'(\alpha)
        & = \nodes{G} \cap \bigcap_{\beta \in \sub{\alpha}} \iskstable(\beta) \\
        & = \{ n \in \nodes{G} \mid n \in \iskstable(\beta) \text{ for all } \beta \in \sub{\alpha}\} \\
        & = \{ n \in \nodes{G} \mid \beta \text{ is } k\text{-stable on } (G,\valuation,n) \\
          & \phantom{= \{ n \in \nodes{G} \mid } \ \text{ for all } \beta \in \sub{\alpha}\} \\
          & = \{ n \in \nodes{G} \mid \alpha \text{ is } k\text{-stable on } (G,\valuation,n) \}.
        \end{align*}
        The first equality is by definition of $\iskstable'(\alpha)$, the third because $\conf$ tracks stability and $\sub{\alpha} \subseteq \isvalid$; and the fourth by definition of $k$-stability.

      \item If $\alpha$ is a fixpoint formula, $\alpha = \pi X. \psi$ then
        $\alpha = \binds{\varphi}{X}$. Since $\alpha \in \isvalid'$,
        $\psi \in \isvalid$ and $\counter(\alpha) = k-1$. Recall that $\conf'$ is sound and
        consistent by
        Lemma~\ref{lem:trans-1-preservation-of-soundness-and-consistency}. Therefore,
        \begin{equation}
          \label{eq:lem:trans-1-preservation-full:eq1}
          \valuation(X) = \sem{\detailadorn{\binds{\varphi}{X}}{C}{k}}{G}{\valuation} = \sem{\detailadorn{\alpha}{C}{k}}{G}{\valuation} = \sem{\detailadorn{\alpha}{k-1}{k}}{G}{\valuation}
        \end{equation}
        Moreover,
        \begin{align}
          \label{eq:lem:trans-1-preservation-full:eq2}
          \result'(\psi) & = \sem{\adorn{\psi}{k}}{G}{\valuation} \nonumber\\
          & = \sem{\adorn{\psi}{k}}{G}{\valuation[X \mapsto \sem{\detailadorn{\alpha}{C}{k}}{G}{\valuation}]} \nonumber\\
          & = \sem{\adorn{\psi}{k}}{G}{\valuation[X \mapsto \sem{\detailadorn{\alpha}{k-1}{k}}{G}{\valuation}]} \nonumber\\
          & = \sem{\pi^k X. \adorn{\psi}{k}}{G}{\valuation} \nonumber \\
          & = \sem{\adorn{\alpha}{k}}{G}{\valuation}
        \end{align}
        The second equality is by soundness of $\conf$, the third by
        \ref{eq:lem:trans-1-preservation-full:eq1}, the fourth by semantics of adorned
        formulae, and the last by definition of uniform adornment.

        Then we reason as follows.
        \begin{align*}
          \iskstable'& (\pi X. \psi) \\
          & = \iskstable(\psi) \cap \isckstable(\pi X.\psi) \\
          & \qquad \cap \{ n \in G \mid n \in \valuation(X) \text{ iff } n \in \result'(\psi) \} \\
          & = \{ n \in \nodes{G} \mid n \in \iskstable(\psi), n \in \isckstable(\pi X.\psi), \\
            & \phantom{= \{ n \in \nodes{G} \mid\ }
            \ n \in \valuation(X) \text{ iff } n \in \result'(\psi) \} \\
            & = \{ n \in \nodes{G} \mid \psi \text{ is } k\text{-stable on } (G,\valuation,n) \\
              & \phantom{= \{ n \in \nodes{G} \mid\ }
                \ \pi X.\psi \text{ is } (\counter,k)\text{-stable on } (G,\valuation,n) \\
                & \phantom{= \{ n \in \nodes{G} \mid\ }
                \  n \in \valuation(X) \text{ iff } n \in \result'(\psi) \} \\
                & = \{ n \in \nodes{G} \mid \psi \text{ is } k\text{-stable on } (G,\valuation,n) \\
                  & \phantom{= \{ n \in \nodes{G} \mid\ }
                    \ \pi X.\psi \text{ is } (\counter,k)\text{-stable on } (G,\valuation,n) \\
                    & \phantom{= \{ n \in \nodes{G} \mid\ }
                      \  n \in \sem{\detailadorn{\alpha}{k-1}{k}}{G}{\valuation} \text{ iff }
                    n \in \sem{\adorn{\alpha}{k}}{G}{\valuation} \} \\
                    & = \{ n \in \nodes{G} \mid \pi X. \psi \text{ is } k\text{-stable on } (G,\valuation,n) \}
                  \end{align*}
                  The first equality is by definition; the second by straightforward rewriting; the third because $\conf$ tracks stability and $\psi \in \isvalid$; the fourth by \eqref{eq:lem:trans-1-preservation-full:eq1} and \eqref{eq:lem:trans-1-preservation-full:eq2}; and the last by Lemma~\ref{lem:c-k-stability-implies-k-stability-when-nested-fixpoints-ready}. \qedhere
              \end{itemize}
            \end{proof}

            To define the second type of transition, we require the notion of a ticking
            fixpoint formula. We say that $\alpha \in \rfp{\varphi}$ \emph{ticks in
            configuration $\conf$} if (1) $\sub{\alpha} \subseteq \isvalid$; (2)
            $\counter(\alpha) < k - 1$; and (3) $\counter(\beta) = k-1$ for every
            $\beta \in \tfp{\alpha}$. We write $\ticks{\conf}$ for the subset of
            $\rfp{\varphi}$ that tick in $\conf$.  Note that if $\alpha$ ticks, none of
            its strict fixpoint subformulae can tick. We observe:

            \begin{lemmarep}
              \label{lem:ticks-implies-counter+1-iterations-done}
              If $\alpha = \pi X. \psi$ ticks in
              coherent configuration $\conf$ then
              $\sem{\detailadorn{\alpha}{\counter(\alpha)+1}{k}}{G}{\valuation} = \result(\psi)$.
            \end{lemmarep}
            \begin{proof}
              Let $\conf = (G, k, \counter, \valuation, \result, \isvalid, \iskstable,\isckstable)$ be a coherent configuration in which $\alpha = \pi X. \psi$ ticks. We want to show that \(\sem{\detailadorn{\alpha}{\counter(\alpha)+1}{k}}{G}{\valuation} = \result(\psi)\). We begin by expanding the left hand side of the equivalence.
              \begin{align*}
                \sem{\detailadorn{\alpha}{\counter(\alpha)+1}{k}}{G}{\valuation} &= \sem{\pi^{\counter(\alpha)+1} X. \adorn{\psi}{k}}{G}{\valuation} \\
                &= \sem{\adorn{\psi}{k}}{G}{\valuation[X \mapsto \sem{\pi^{\counter(\alpha)} X. \adorn{\psi}{k}}{G}{\valuation}]}
              \end{align*}
              Let us shorten \(\valuation' = \valuation[X \mapsto \sem{\pi^{\counter(\alpha)} X. \adorn{\psi}{k}}{G}{\valuation}]\), so we get \(\sem{\detailadorn{\alpha}{\counter(\alpha)+1}{k}}{G}{\valuation} = \sem{\adorn{\psi}{k}}{G}{\valuation'}\).

              Since \(\alpha \in \ticks{\conf}\), we find \(\psi \in \isvalid\). As \(\conf\) is coherent, this implies \(\result(\psi) = \sem{\adorn{\psi}{k}}{G}{\valuation}\). Hence, we have to show that \(\sem{\adorn{\psi}{k}}{G}{\valuation'} = \sem{\adorn{\psi}{k}}{G}{\valuation}\). Since the semantics of a formula depend only on the valuation of the free variables, we can restrict our attention to the free variables of \(\psi\). In this case, \(\free{\psi} = \free{\alpha} \cup \{X\}\).
              \begin{itemize}
                \item For all \(Y \in \free{\alpha}\), we find \(\valuation(Y) = \valuation'(Y)\) by the definition of \(\valuation'\).
                \item For \(X\), we find \(\valuation'(X) = \sem{\pi^{\counter(\alpha)} X. \adorn{\psi}{k}}{G}{\valuation}\). Because \(\conf\) is coherent, and thus sound, we find \(\valuation(X) = \sem{\pi^{\counter(\alpha)} X. \adorn{\psi}{k}}{G}{\valuation}\). This also equals \(\valuation'(X) = \valuation(X)\).
              \end{itemize}
              Hence, we find that \(\valuation'(Y) = \valuation(Y)\) for all \(Y \in \free{\psi}\). This implies that \(\sem{\adorn{\psi}{k}}{G}{\valuation'} = \sem{\adorn{\psi}{k}}{G}{\valuation}\), and thus \(\sem{\detailadorn{\alpha}{\counter(\alpha)+1}{k}}{G}{\valuation} = \result(\psi)\).
            \end{proof}

            This  lemma shows that if  $\alpha = \pi X. \psi$ ticks in a coherent $\conf$ then in $\result(\psi)$ we have already computed $\counter(\alpha)+1$ iterations of $\alpha$'s outermost fixpoint. However, since $\counter(\alpha) +1 <  k$, we have not yet computed all necessary $k$ fixpoint iterations of $\adorn{\alpha}{k} = \detailadorn{\alpha}{k}{k}$. To ensure that computation can continue, a type-2 transition therefore changes the configuration such that it causes $\sem{\detailadorn{\alpha}{\counter(\alpha)+2}{k}}{G}{\valuation}$ to be computed in further transition steps. It does so by copying $\result(\psi) = \sem{\detailadorn{\alpha}{\counter(\alpha)+1}{k}}{G}{\valuation}$ to $\valuation(X)$, increasing $\counter(\alpha)$, and resetting the computation of all subformulae that depend on the value of $X$, which has now changed.

            Formally, define  $\reset{\conf}$ to be the smallest subset of $\vars{\varphi}$ satisfying
            \begin{multline*}
              \reset{\conf}  = \{X \mid \binds{\varphi}{X} \in \ticks{\conf}\} 
              \\ \cup \{ Y \mid \free{\binds{\varphi}{Y}} \cap \reset{\conf} \not = \emptyset\}.
            \end{multline*}
            Define  $\dep{\conf} \defeq \{ \binds{\varphi}{X} \mid X \in \reset{\conf}\} \setminus \ticks{\conf}$. We say that the elements of $\dep{\conf}$ \emph{depend} on a tick in $\conf$.

            \begin{definition}
              \label{def:trans-2}
              Let $\conf = (G, k, \counter, \valuation, \result, \isvalid, \iskstable,\isckstable)$ be a configuration. A \emph{type-2} transition on $\conf$ yields the configuration $\conf' = (G, k, \counter', \valuation', \result, \isvalid', \iskstable, \isckstable')$ where
              \begin{align*}
                \counter'(\alpha) &\defeq
                \begin{cases}
                  \counter(\alpha) + 1 & \text{if } \alpha \in \ticks{\conf}\\
                  0 & \text{if } \alpha \in \dep{\conf} \\
                  \counter(\alpha) & \text{otherwise}\\
                \end{cases} \\
                \valuation'(X) & \defeq
                \begin{cases}
                  \result(\psi) & \text{if } \binds{\varphi}{X} \in \ticks{\conf},  \binds{\varphi}{X}= \pi X.\psi\\
                  \emptyset & \text{if } \binds{\varphi}{X} \in \dep{\conf} \cap \rlfp{\varphi} \\
                  \nodes{G} & \text{if } \binds{\varphi}{X} \in \dep{\conf} \cap \rgfp{\varphi} \\
                  \valuation(X) &\text{otherwise}
                \end{cases} \\
\isvalid' & \defeq \{ \alpha \in \isvalid \mid \free{\alpha} \cap \reset{\conf} = \emptyset \} \\
\isckstable'(\pi X. \psi) & \defeq
                \begin{cases}
                  \isckstable(\pi X. \psi) \cap \iskstable(\psi) & \text{ if } \pi X. \psi \in \ticks{\conf} \\
                  \nodes{G} & \text{ if } \pi X. \psi \in \dep{\conf} \\
                  \isckstable(\pi X. \psi) & \text{ otherwise}
                \end{cases}
\end{align*}
We denote this by $\conf \trans{2} \conf'$.
            \end{definition}

            Note that if no fixpoint formula ticks in $\conf$, then $\conf' = \conf$, i.e., the transition is a no-op.

            \begin{toappendix}
              The following lemma shows that a type-2 transition changes the valuation only for those variables that do not reset.
\begin{lemma}
  \label{lem:trans-2-equal-vars}
  If $\conf \trans{2} \conf'$ then $\valuation'(X) = \valuation(X)$ for all $X \in \vars{\varphi} \setminus \reset{\conf}$, where $\valuation$ and $\valuation'$ are the valuations of $\conf$ and $\conf'$, respectively.
\end{lemma}
\begin{proof}
  Notation-wise, assume $(G, k, \counter, \valuation, \result, \isvalid, \iskstable, \isckstable) = \conf$ and
  $(G, k, \counter', \valuation', \result, \isvalid', \iskstable, \isckstable') = \conf'$.
  By definition of $\reset{\conf}$ we have for all variables $X$ that $X \in \reset{\conf}$ iff $\binds{\varphi}{X} \in \ticks{\conf}$ or $\binds{\varphi}{X} \in \dep{\conf}$.
  Let $X \in \vars{\varphi} \setminus \reset{\conf}$.
  Then $\binds{\varphi}{X} \not \in \ticks{\conf}$ and $\binds{\varphi}{X} \not \in \dep{\conf}$. Hence, by definition of $\valuation'$, $\valuation'(X) = \valuation(X)$.
\end{proof}

The following lemma shows that the only variables that can reset belong to
fixpoint formulae that are (not necessarily strict) subformulae of ticking
fixpoints.
\begin{lemma}
  \label{lem:deps}
  If $X \in \reset{\conf}$ then $\binds{\varphi}{X} \in \rfp{\ticks{\conf}}$.
\end{lemma}
\begin{proof}
  The proof is by induction on the argument that shows that $X \in \reset{\conf}$.
  \begin{itemize}
    \item If $\binds{\varphi}{X} \in \ticks{\conf}$ then the result follows from
      the fact that $\ticks{\conf} \subseteq \rfp{\ticks{\conf}}$.
    \item If $\binds{\varphi}{X} \not \in \ticks{\conf}$ there is some
      $Y \in \free{\binds{\varphi}{X}} \cap \reset{\conf}$. Because $\varphi$ is
      a sentence and well-named, necessarily
      $\binds{\varphi}{X} \in \tfp{\binds{\varphi}{Y}}$. In particular,
      $X \not = Y$. There is hence a shorter argument that shows that
      $Y \in \reset{\conf}$. By induction hypothesis,
      $\binds{\varphi}{Y} \in \rfp{\ticks{\conf}}$. Hence
      $\binds{\varphi}{X} \in \rfp{\ticks{\conf}}$. \qedhere
  \end{itemize}
\end{proof}

\begin{lemma}
  \label{lem:trans-2-equal-vars-2}
  If $\conf \trans{2} \conf'$ then $\valuation'(X) = \valuation(X)$ for all $X \in \free{\ticks{\conf}}$, where $\valuation$ and $\valuation'$ are the valuations of $\conf$ and $\conf'$, respectively.
\end{lemma}
\begin{proof}
  Notation-wise, assume
  $(G, k, \counter, \valuation, \result, \isvalid, \iskstable, \isckstable) = \conf$ and
  $(G, k, \counter', \valuation', \result, \isvalid', \iskstable, \isckstable') = \conf'$.  Let
  $X \in \free{\ticks{\conf}}$. By Lemma~\ref{lem:trans-2-equal-vars}, it
  suffices to show that $X \not \in \reset{\conf}$, since this implies
  $\valuation'(X) = \valuation(X)$.

  We proceed as follows. Suppose, for the purpose of obtaining a contradiction
  that $X \in \reset{\conf}$. By Lemma~\ref{lem:deps},
  $\binds{\varphi}{X} \in \rsub{\ticks{\conf}}$. As such, either
  $\binds{\varphi}{X} \in \ticks{\conf}$ or there is some
  $\beta \in \ticks{\conf}$ with $\binds{\varphi}{X} \in \tfp{\beta}$. We show
  that neither can hold.

  Since $X \in \free{\ticks{\conf}}$ there exists $\alpha \in \ticks{\conf}$
  such that $X \in \free{\alpha}$. Since $\varphi$ is a sentence and is
  well-named, $\alpha \in \rsub{\binds{\varphi}{X}}$.  Because $\alpha$ ticks in
  $\conf$ we know that $\counter(\alpha) < k-1$. As such, $\binds{\varphi}{X}$
  cannot tick in $\conf$: if it did tick, this would require that
  $\counter(\alpha) = k-1$ since $\alpha$ is a strict fixpoint-subformula of
  $\binds{\varphi}{X}$ and every such subformula must be mapped to $k-1$ by
  $\counter$.  By the same reasoning no $\beta \in \ticks{\conf}$ with
  $\binds{\varphi}{X} \in \tfp{\beta}$ can tick, since this would also require
  that $\counter(\alpha) = k-1$ as
  $\alpha \in \rfp{\binds{\varphi}{X}} \subseteq \tfp{\beta}$.

  As such, $X \not \in \free{\ticks{\conf}}$.
\end{proof}

\begin{lemma}
  \label{lem:trans-2-preservation-of-soundness-and-consistency}
  If $\conf \trans{2} \conf'$ and $\conf$ is sound and consistent then so
  is $\conf'$.
\end{lemma}
\begin{proof}
  Notation-wise, assume $(G, k, \counter, \valuation, \result, \isvalid, \iskstable, \isckstable) = \conf$ and
  $(G, k, \counter', \valuation', \result, \isvalid', \iskstable, \isckstable') = \conf'$.

  We first prove soundness. Let $X \in \vars{\varphi}$. We need to prove that $\valuation'(X) = \sem{\detailadorn{\binds{\varphi}{X}}{\counter}{k}}{G}{\valuation'}$. Assume that $\binds{\varphi}{X} = \pi X. \psi$. We distinguish four cases.
  \begin{enumerate}
    \item $\binds{\varphi}{X} \in \ticks{\conf}$. By definition of $\counter'$,
      $\counter'(\binds{\varphi}{X}) = \counter(\binds{\varphi}{X}) + 1$. By
      definition of tick, $\psi \in \sub{\binds{\varphi}{X}} \subseteq \isvalid$
      and $\counter(\binds{\varphi}{X}) < k-1$.
      By Lemma ~\ref{lem:trans-2-equal-vars-2}, $\valuation'(Y) = \valuation(Y)$ for every $Y \in \free{\binds{\varphi}{X}} = \free{\psi} \setminus \{X\}$.
      Then
      \begin{align*}
        \valuation'(X)
         = \result(\psi) \\
        & = \sem{\adorn{\psi}{k}}{G}{\valuation} \\
        & = \sem{\adorn{\psi}{k}}{G}{\valuation[X \mapsto \sem{\detailadorn{\binds{\varphi}{X}}{C}{k}}{G}{\valuation}]} \\
        & = \sem{\adorn{\psi}{k}}{G}{\valuation'[X \mapsto \sem{\detailadorn{\binds{\varphi}{X}}{\counter}{k}}{G}{\valuation}]} \\
        & = \sem{\adorn{\psi}{k}}{G}{\valuation'[X \mapsto \sem{\detailadorn{\binds{\varphi}{X}}{\counter}{k}}{G}{\valuation'}]} \\
        & = \sem{\mu^{\counter(\binds{\varphi}{X})+1} X. \psi}{G}{\valuation'} \\
        & = \sem{\detailadorn{\binds{\varphi}{X}}{\counter'}{k}}{G}{\valuation'}
      \end{align*}
      The first equality is by definition of $\valuation'$;  the second because $\conf$ is consistent and $\psi \in \isvalid$; the third because $\valuation(X) = \sem{\detailadorn{\binds{\varphi}{X}}{\counter}{k}}{G}{\valuation}$ as $\conf$ is sound; the fourth because $\valuation$ and $\valuation'$ agree on all $Y \in \free{\psi} \setminus \{X\}$ and the semantics only depends on the value of free variables; the fifth because using the same reasoning we see that $\sem{\detailadorn{\binds{\varphi}{X}}{\counter}{k}}{G}{\valuation} = \sem{\detailadorn{\binds{\varphi}{X}}{\counter}{k}}{G}{\valuation'}$; the sixth by definition of approximate semantics; and the last by definition.

    \item $\binds{\varphi}{X} \in \dep{\conf} \cap \rlfp{\ticks{\conf}}$. Then  $\counter'(\binds{\varphi}{X}) = 0$ by definition of $\counter'$. Consequently  $\sem{\detailadorn{\binds{\varphi}{X}}{\counter'}{k}}{G}{\valuation'} = \emptyset = \valuation'(X)$, as desired.

    \item $\binds{\varphi}{X} \in \dep{\conf} \cap \rgfp{\ticks{\conf}}$. Similar to the previous case.

    \item  None of the above hold. In this case, $\counter'(\binds{\varphi}{X}) = \counter(\binds{\varphi}{X})$ and $\valuation'(X) = \valuation(X)$. Moreover,
      $\binds{\varphi}{X} \not \in \ticks{\conf}$ and $\binds{\varphi}{X} \not \in \dep{\conf} \cap \rfp{\varphi}$. Since $\binds{\varphi}{X} \in \rfp{\varphi}$ it follows that $\binds{\varphi}{X} \not \in \dep{\conf}$. By definition of $\reset{\conf}$, it holds that $X \in \reset{\conf}$ iff $X \in \ticks{\conf}$ or $X \in \dep{\conf}$. Hence, $X \not \in \reset{\conf}$. Again by definition of reset, $\free{\binds{\varphi}{X}} \cap \reset{\conf} = \emptyset$. (If it were non-empty, $X$ would be in $\reset{\conf}$.) Hence,  $\valuation'(Y) = \valuation(Y)$ for every $Y \in \free{\binds{\varphi}{X}}$ by Lemma~\ref{lem:trans-2-equal-vars}
      Then,
      \begin{align*}
        \valuation'(X)
        & = \valuation(X) \\
        & = \sem{\detailadorn{\binds{\varphi}{X}}{\counter}{k}}{G}{\valuation} \\
        & = \sem{\detailadorn{\binds{\varphi}{X}}{\counter}{k}}{G}{\valuation'} \\
        & = \sem{\detailadorn{\binds{\varphi}{X}}{\counter'}{k}}{G}{\valuation'}
      \end{align*}
      The first equality is by definition of $\valuation'$; the second because $\conf$ is sound; and the third because the semantics only depends on the values of free variables, and $\valuation'$ and $\valuation$ agree on the free variables of $\binds{\varphi}{X}$; and the last because $\counter'(\binds{\varphi}{X}) = \counter(\binds{\varphi}{X})$.
  \end{enumerate}

  It remains to prove consistency. We need show that for all $\alpha \in \isvalid'$:
  \begin{enumerate}
    \item $\result'(\alpha) = \sem{\adorn{\alpha}{k}}{G}{\valuation'}$;
    \item $\sub{\alpha} \subseteq \isvalid'$;
    \item if $\alpha$ is a fixpoint formula, then $\counter'(\alpha) = k-1$.
  \end{enumerate}
  Let $\alpha \in \isvalid'$. By definition of $\isvalid'$, $\alpha \in \isvalid$ and $\free{\alpha} \cap \reset{\conf} = \emptyset$.

  We first show item (1).  By Lemma~\ref{lem:trans-2-equal-vars},
  $\valuation'(X) = \valuation(X)$ for every $X \in \free{\alpha}$. Therefore,
  $\result(\alpha) = \sem{\adorn{\alpha}{k}}{G}{\valuation} =
  \sem{\adorn{\alpha}{k}}{G}{\valuation'}$, where the first equality is due to
  consistency of $\conf$ and the fact that $\alpha \in \isvalid$, and the second
  due to the fact that the semantics only depends on the values of free
  variables, on which $\valuation$ and $\valuation'$ agree.

  We next show item (2). Let $\alpha \in \isvalid' \subseteq \isvalid$. Because $\alpha \in \isvalid$ and $\conf$ is consistent, we know that $\sub{\alpha} \subseteq \isvalid$. Assume for the purpose of obtaining a contradiction, that some $\beta \in \sub{\alpha}$ is not in $\isvalid'$. Then $\beta \in \isvalid$ and there exists some $X \in \free{\beta} \cap \reset{\conf}$.
  There are two possibilities.

  \begin{itemize}
    \item  $X$ is also free in $\alpha$. But then $X \in \free{\alpha} \cap \reset{\conf}$, which contradicts $\alpha \in \isvalid'$.
    \item $X$ is not free in $\alpha$. The only way that this can happen is if $\alpha = \binds{\varphi}{X}$. But since $X \in \reset{\conf}$, this means that either $\alpha = \binds{\varphi}{X} \in \ticks{\conf}$ or $\alpha = \binds{\varphi}{X} \in \dep{\conf}$. We show that neither can happen.

      If $\alpha \in \ticks{\conf}$ then $\counter(\alpha) < k-1$, contradicting the fact that $\alpha \in \isvalid$ by consistency of $\conf$ (property (3)).

      If $\alpha \not \in \ticks{\conf}$ but $\alpha \in \dep{\conf}$, then there has to be some other variable $Y$ in $\reset{\conf}$ that is free in $\alpha$. But then $\free{\alpha} \cap \reset{\conf} \not = \emptyset$, contradicting that $\alpha \in \isvalid'$.
  \end{itemize}

  Finally, we show item (3). Assume that additionally $\alpha$ is a fixpoint
  formula, so $\alpha = \binds{\varphi}{X}$ for some variable $X$. Because $\conf$
  is consistent, because $\alpha$ is a fixpoint formula, and because
  $\alpha \in \isvalid' \subseteq \isvalid$ we know that $\counter(\binds{\varphi}{X}) = k-1$. So,
  $\binds{\varphi}{X} \not \in \ticks{\conf}$. Furthermore, by definition of
  $\isvalid'$, $\free{\binds{\varphi}{X}} \cap \reset{\conf} = \emptyset$. Then,
  the definition of dep yields that $\binds{\varphi}{X} \not \in
  \dep{\conf}$. Hence, by definition of $\counter'$ we obtain that $\counter'(\alpha) = \counter(\alpha) = k-1$.
\end{proof}

\begin{lemma}
  \label{lem:k-stable-depends-only-on-free-vars}
  If $\alpha \in \rsub{\varphi}$ is $k$-stable on $(G,\valuation,n)$ and
  $\valuation'$ agrees with $\valuation$ on $\free{\alpha}$ then $\alpha$ is
  $k$-stable on $(G,\valuation',n)$.
\end{lemma}
\begin{proof}
  This is essentially because the approximation semantics only depends on the
  free variables of $\alpha$, on which $\valuation$ and $\valuation'$ agree.
  The formal proof is by induction on $\alpha$.
  \begin{itemize}
    \item When $\alpha = p, \neg p$ or $X$, the proof is immediate.
    \item When $\alpha$ is another non-fixpoint formula, then all subformulae $\psi \in \sub{\alpha}$ are $k$-stable on $(G,\valuation,n)$. Since they have the same free variables as $\alpha$,  they are also $k$-stable on $(G, \valuation',n)$ by induction hypothesis. Hence, $\alpha$ is $k$-stable on $(G,\valuation',n)$.
    \item When $\alpha$ is a fixpoint formula $\pi X. \psi$, then
      \begin{enumerate}
        \item $\sem{\detailadorn{\alpha}{k}{k}}{G}{V} = \sem{\detailadorn{\alpha}{k-1}{k}}{G}{V}$; and
        \item for every $0 \leq i < k$, $\psi$ is $k$-stable on $(G,V_i,n)$ where $V_i = V[X \mapsto \sem{\detailadorn{\alpha}{i}{k}}{G}{V}]$.
      \end{enumerate}
      Therefore,
      from (1) and the fact that the approximate semantics only depends on the values of free variables, on which $\valuation$ and $\valuation'$ agree, we obtain that
      \begin{align*}
        \sem{\detailadorn{\alpha}{k}{k}}{G}{\valuation'} = \sem{\detailadorn{\alpha}{k}{k}}{G}{\valuation} =  \sem{\detailadorn{\alpha}{k-1}{k}}{G}{\valuation} =  \sem{\detailadorn{\alpha}{k-1}{k}}{G}{\valuation'}.
      \end{align*}
      By the same reasoning,
      $\sem{\detailadorn{\alpha}{i}{k}}{G}{V'} =
      \sem{\detailadorn{\alpha}{i}{k}}{G}{V}$ for $0 \leq i < k$. Hence, if we
      define
      $\valuation'_i = V'[X \mapsto \sem{\detailadorn{\alpha}{i}{k}}{G}{V'}]$ for
      $0 \leq i < k$ then $\valuation_i$ and $\valuation'_i$ agree on
      $\free(\psi) = \free{\alpha} \cup\{X\}$. From (2) and the induction
      hypothesis, we then obtain that $\psi$ is $k$-stable on $(G,\valuation'_i,n)$
      for $1 \leq i < k$. We may hence conclude that $\varphi$ is $k$-stable on
      $(G, \valuation',n)$. \qedhere
  \end{itemize}
\end{proof}

\begin{lemma}
  \label{lem:c-k-stability-depends-only-on-free-vars}
  If $\alpha \in \rfp{\varphi}$ is $(\counter,k)$-stable on $(G,\valuation,n)$ and
  $\valuation'$ agrees with $\valuation$ on $\free{\alpha}$ then $\alpha$ is
  $(\counter,k)$-stable on $(G,\valuation',n)$.
\end{lemma}
\begin{proof}
  Assume $\alpha = \pi X. \psi$. Since $\alpha$ is $(\counter,k)$-stable on $(G,\valuation,n)$, we know that $\psi$ is $k$-stable on $(G,\valuation_i,n)$ for $0 \leq i < \counter(\alpha)$ with $\valuation_i = \valuation[X \mapsto \sem{\detailadorn{\alpha}{i}{k}}{G}{\valuation}]$. Because the approximate semantics only depends on the values of free variables, on which $\valuation$ and $\valuation'$ agree, it follows that $\sem{\detailadorn{\alpha}{i}{k}}{G}{\valuation} = \sem{\detailadorn{\alpha}{i}{k}}{G}{\valuation'}$. Define $\valuation'_i = \valuation'[X \mapsto \sem{\detailadorn{\alpha}{i}{k}}{G}{\valuation'}]$. Then, because $\valuation'$ and $\valuation$ agree on $\free{\alpha}$, $\valuation_i$ and $\valuation'_i$ agree on $\free{\alpha} \cup \{X\}$, i.e., on all free variables of $\psi$. Hence, by Lemma~\ref{lem:k-stable-depends-only-on-free-vars}, $\psi$ is $k$-stable on $(G,\valuation'_i,n)$, for all $0 \leq i < \counter(\alpha)$. As such, $\alpha$ is $(\counter, k)$-stable on $(G, \valuation',n)$.
\end{proof}

             \end{toappendix}

            \begin{lemmarep}
              \label{lem:trans-2-preservation-full}
              If $\conf \trans{2} \conf'$ and $\conf$ is coherent then so
              is $\conf'$.
            \end{lemmarep}
            \begin{proof}
              Notation-wise, assume $(G, k, \counter, \valuation, \result, \isvalid, \iskstable, \isckstable) = \conf$ and $(G, k, \counter', \valuation', \result, \isvalid', \iskstable, \isckstable') = \conf'$.  By
              Lemma~\ref{lem:trans-2-preservation-of-soundness-and-consistency} we obtain that $\conf'$ is sound and consistent. It hence remains to prove that $\conf'$ tracks stability, for which we need to show that
              \begin{enumerate}
                \item For every $\alpha \in \isvalid'$, it holds that $n \in \iskstable(\alpha) \iff \alpha \text{ is } k\text{-stable on } (G,\valuation',n)$; and
                \item For every $\alpha \in \rfp{\varphi}$, it holds that $n \in \isckstable'(\alpha) \iff \alpha \text{ is } (\counter',k)\text{-stable on } (G,\valuation',n)$.
              \end{enumerate}
              We first prove item (1). Let $\alpha \in \isvalid'$, and $n$ be an arbitrary node in $G$. Let us start with the forward direction of the equivalence. We need to show that if $n \in \iskstable(\alpha)$ then $\alpha$ is $k$-stable on $(G,\valuation',n)$. Because $\isvalid' \subseteq \isvalid$, we have $\alpha \in \isvalid$.  Because $\conf$ is coherent, it tracks stability, and thus $\alpha$ is $k$-stable on $(G,\valuation,n)$.
              By the definition of $\isvalid'$, we have $\free{\alpha} \cap \reset{\conf} = \emptyset$. Combined with Lemma~\ref{lem:trans-2-equal-vars}, $\valuation'(Y) = \valuation(Y)$ for all $Y \in \free{\alpha}$. Hence, by Lemma~\ref{lem:k-stable-depends-only-on-free-vars}, $\alpha$ is $k$-stable on $(G,\valuation',n)$.

              For the converse direction, we need to show that if $\alpha$ is $k$-stable on $(G,\valuation',n)$ then $n \in \iskstable(\alpha)$. We have already established that $\valuation'(Y) = \valuation(Y)$ for all $Y \in \free{\alpha}$. Hence, by Lemma~\ref{lem:k-stable-depends-only-on-free-vars}, $\alpha$ is $k$-stable on $(G,\valuation,n)$. Because $\conf$ is coherent, and tracks stability, we conclude that $n \in \iskstable(\alpha)$.

              We now continue with the proof of item (2). Let $\alpha \in \rfp{\varphi}$, $\alpha = \pi X.\psi$, and $n$ be an arbitrary node in $G$. We start with the backward direction of the equivalence. We need to show that if $\alpha$ is $(\counter',k)$-stable on $(G,\valuation',n)$, then $n \in \isckstable'(\alpha)$. We distinguish three cases, aligning with the definition of $\isckstable'$:

              \begin{itemize}
                \item $\alpha \in \ticks{\conf}$. As $\alpha \in \ticks{\conf}$, we have $\counter'(\alpha) = \counter(\alpha) + 1$. Therefore, $\alpha$ is $(\counter(\alpha)+1,k)$-stable on $(G,\valuation',n)$. By the definition of \trans{2}, we have $\isckstable'(\alpha) = \isckstable(\alpha) \cap \iskstable(\psi)$. It suffices to show that $n \in \isckstable(\alpha)$ and $n \in \iskstable(\psi)$.

                  We first show that $n \in \isckstable(\alpha)$. Since $\alpha \in \ticks{\conf}$, $\free{\alpha}$ is a subset of $\free{\ticks{\conf}}$. By Lemma~\ref{lem:trans-2-equal-vars-2}, $\valuation'(Y) = \valuation(Y)$ for all $Y \in \free{\alpha}$. Lemma~\ref{lem:c-k-stability-depends-only-on-free-vars} then implies that $\alpha$ is $(\counter(\alpha)+1,k)$-stable on $(G,\valuation,n)$. By Lemma~\ref{lem:j-k-stability-implies-smaller-j-k-stability}, it is also $(\counter(\alpha),k)$-stable on $(G,\valuation,n)$. Since $\conf$ is coherent and tracks stability, we conclude that $n \in \isckstable(\alpha)$.

                  We next show that $n \in \iskstable(\psi)$. By Definition~\ref{def:c-k-stability}, for any integer $j$ so that $0 \leq j < \counter(\alpha) + 1$, the body $\psi$ is $k$-stable on $(G,\valuation'_j,n)$, where $\valuation'_j = \valuation'[X \mapsto \sem{\detailadorn{\alpha}{j}{k}}{G}{\valuation'}]$. In particular, it means that $\psi$ is $k$-stable on $(G,\valuation'_{\counter(\alpha)},n)$. We will show that $\valuation'_{\counter(\alpha)}$ and $\valuation$ are effectively equal, for the purpose of evaluating $\psi$. The free variables of $\psi$ are $\free{\alpha} \cup \{X\}$.

                  By Lemma~\ref{lem:trans-2-equal-vars-2}, $\valuation'(Y) = \valuation(Y)$ for all $Y \in \free{\alpha}$. We know $\sem{\detailadorn{\alpha}{j}{k}}{G}{\valuation'}$ only depends on $\valuation'$ restricted to $\free{\alpha}$, and therefore $\sem{\detailadorn{\alpha}{j}{k}}{G}{\valuation'}$ equals $\sem{\detailadorn{\alpha}{j}{k}}{G}{\valuation}$, which by definition equals $\valuation(X)$. Hence, $\valuation'_{\counter(\alpha)} = \valuation'[X \mapsto \valuation(X)]$. Let us now compare $\valuation'_{\counter(\alpha)}$ and $\valuation$. For $X$ they are equal, since $\valuation'_{\counter(\alpha)}(X) = \valuation'[X \mapsto \valuation(X)] = \valuation(X)$. For $Y \in \free{\alpha}$, we already have shown that $\valuation'(Y) = \valuation(Y)$. That covers all free variables of $\psi$, thus $\valuation'_{\counter(\alpha)} = \valuation$ for all free variables in $\psi$.

                  Remember that $\psi$ is $k$-stable on $(G,\valuation'_{\counter(\alpha)},n)$, and together Lemma~\ref{lem:k-stable-depends-only-on-free-vars}, it follows that $\psi$ is $k$-stable on $(G,\valuation,n)$. Since $\conf$ is coherent and tracks stability, we conclude that $n \in \iskstable(\psi)$.

                \item $\alpha \in \dep{\conf}$. When $\alpha$ is in $\dep{\conf}$, then all vertices are in $\isckstable'(\alpha)$, by definition of $\isckstable'$. Hence, $n \in \isckstable'(\alpha)$.

                \item $\alpha \not \in \ticks{\conf}$ and $\alpha \not \in \dep{\conf}$. By
                  definition of $\counter'$ we have $\counter'(\alpha) =
                  \counter(\alpha)$. So, $\alpha$ is $(\counter(\alpha),k)$-stable on
                  $(G,\valuation,n)$.  Because $\alpha$ is neither in $\ticks{\conf}$ nor in
                  $\dep{\conf}$, it follows that
                  $\free{\alpha} \cap \reset{\conf} = \emptyset$. Then, by
                  Lemma~\ref{lem:trans-2-equal-vars}, $\valuation'(Y) = \valuation(Y)$ for
                  all $Y \in \free{\alpha}$. Hence, by
                  Lemma~\ref{lem:c-k-stability-depends-only-on-free-vars}, is
                  $(\counter,k)$-stable on $(G,\valuation,n)$. Since $\conf$ tracks
                  stability, we conclude that $n \in \isckstable(\alpha)$. Then, because
                  $\isckstable'(\alpha) = \isckstable(\alpha)$ by definition of
                  $\isckstable'$, $n \in \isckstable'(\alpha)$, as desired.
              \end{itemize}

              We now prove the converse direction. Assume that $n \in \isckstable'(\alpha)$. We need to show that $\alpha$ is $(\counter',k)$-stable on $(G,\valuation',n)$. Once more, we distinguish three cases based on the definition of $\isckstable'$:

              \begin{itemize}
                \item $\alpha \in \ticks{\conf}$. By definition of $\isckstable'$, we have $n \in \isckstable(\alpha)$ and $n \in \iskstable(\psi)$. By definition of $\counter'$, $\counter'(\alpha) = \counter(\alpha) + 1$. Furthermore, by definition of ticking, it must be the case that $\counter(\alpha) = k-1$. Therefore, $\counter'(\alpha) = k$. Hence, what we really need to show here is that $\alpha$ is $(k,k)$-stable on $(G,\valuation',n)$. Using Lemma~\ref{lem:trans-2-equal-vars-2}, we find that $\valuation'(Y) = \valuation(Y)$ for all $Y \in \free{\alpha}$. Therefore, by Lemma~\ref{lem:c-k-stability-depends-only-on-free-vars}, $\alpha$ is $(k,k)$-stable on $(G,\valuation',n)$ if, and only if, $\alpha$ is $(k,k)$-stable on $(G,\valuation,n)$. It hence suffices to show that $\alpha$ is $(k,k)$-stable on $(G,\valuation,n)$.
                  By Definition~\ref{def:c-k-stability}, this means we must show that for every integer $j$ so that $0 \leq j < k$, the body $\psi$ is $k$-stable on $(G,\valuation_j,n)$, where $\valuation'_j = \valuation[X \mapsto \sem{\detailadorn{\alpha}{j}{k}}{G}{\valuation}]$. We do so as follows.

                  We know that $n \in \isckstable(\psi)$, which since $\conf$ is tracks stability, implies that $\alpha$ is $(\counter,k)$-stable on $(G,\valuation,n)$. This means, that for all integers $j$ so that $0 \leq j < \counter(\alpha) = k-1$, $\psi$ is $k$-stable on $(G,\valuation_j,n)$. It hence remains to show that $\psi$ is $k$-stable on $(G,\valuation_{k-1},n)$. Because $\conf$ is sound,
                  \[ \valuation(X) = \sem{\detailadorn{\binds{\varphi}{X}}{\counter}{k}}{G}{\valuation}
                    = \sem{\detailadorn{\alpha}{\counter}{k}}{G}{\valuation}
                  = \sem{\detailadorn{\alpha}{k-1}{k}}{G}{\valuation} \]
                  This implies that $V = V_{k-1}$. Then because $n \in \iskstable(\psi)$, and $\conf$ tracks stability, we derive that $\psi$ is $k$-stable on $\valuation = \valuation_{k-1}$. Hence, $\alpha$ is $(\counter',k)$-stable on $(G,\valuation,n)$.

                \item $\alpha \in \dep{\conf}$. By definition of $\isckstable'$, we have $n \in \isckstable(\alpha)$. As $\alpha$ is in $\dep{\conf}$, we have $\counter'(\alpha) = 0$. Hence, we need to show that $\alpha$ is $(0,k)$-stable on $(G,\valuation',n)$. According to Definition~\ref{def:c-k-stability}, this is vacuously true, since there are no integers $j$ such that $0 \leq j < 0$. Hence, $\alpha$ is $(\counter',k)$-stable on $(G,\valuation',n)$.
                \item $\alpha \not \in \ticks{\conf}$ and $\alpha \not \in \dep{\conf}$. By definition of $\isckstable'$, we have $\isckstable'(\alpha) = \isckstable(\alpha)$. Similarly to before, we find that $\counter'(\alpha) = \counter(\alpha)$, and $\free{\alpha} \cap \reset{\conf} = \emptyset$.

                  By substituting the counter, what we need to show becomes that $\alpha$ is $(\counter,k)$-stable on $(G,\valuation',n)$. Because $\conf$ is coherent and $n \in \isckstable(\alpha)$, we know that $\alpha$ is $(\counter,k)$-stable on $(G,\valuation,n)$. Through Lemma~\ref{lem:trans-2-equal-vars} we find that $\valuation'(Y) = \valuation(Y)$ for all $Y \in \free{\alpha}$. Combined with Lemma~\ref{lem:c-k-stability-depends-only-on-free-vars}, we find that $\alpha$ is $(\counter,k)$-stable on $(G,\valuation',n)$. \qedhere
              \end{itemize}
            \end{proof}

            The third kind of transition  increases the bound $k$ when $\conf$ is complete.

            \begin{definition}
              \label{def:trans-3}
              Let $\conf$ be a configuration with graph $G$ and bound $k$. A
              \emph{type-3} transition on $\conf$ yields the configuration $\conf'$ such that $\conf = \conf'$ if $\conf$ is not
              complete. Otherwise, $\conf'$ is the initial configuration of
              $\varphi$ on $G$ w.r.t. $k+1$. We write $\conf \trans{3} \conf'$
              to indicate that $\conf'$ is the type-3 transition of $\conf$.
            \end{definition}

It is straightforward to show:
            \begin{lemmarep}
              \label{lem:trans-3-preservation}
              If $\conf \trans{3} \conf'$ and $\conf$ is coherent, then so is $\conf'$.
            \end{lemmarep}
            \begin{proof}
              Trivial, since either $\conf'= \conf$, which is coherent by assumption, or $\conf'$ is the initial configuration w.r.t. $k+1$, which is coherent by definition.
            \end{proof}

            Let us write $\trans{1,2}$ for the composition of $\trans{1}$ and $\trans{2}$, with $\trans{2}$ executing after $\trans{1}$. We define $\trans{3,1,2}$ similarly. We use $\trans{1,2}^*$ to denote the reflexive-transitive closure of $\trans{1,2}$ and similarly for $\trans{3,1,2}^*$.

            \begin{toappendix}
              \begin{lemma}
  \label{lem:trans-1-no-op-on-complete}
  If $\conf \trans{1} \conf'$ and $\conf$ is coherent and complete then $\conf = \conf'$.
\end{lemma}
\begin{proof}
  If $\conf$ is complete, then $\isvalid = \rsub{\varphi}$. Since $\isvalid \subseteq \isvalid'$ by Lemma~\ref{lem:trans-1-extends-valid}, also $\isvalid' = \rsub{\varphi}$. Moreover, also $\conf'$ is coherent by Lemma~\ref{lem:trans-1-preservation-full}. Since both $\conf$ and $\conf'$ are consistent and $\isvalid = \isvalid'$ it is necessarily the case that $\result = \result'$. Also, the fact that both $\conf$ and $\conf'$ track stability and $\isvalid = \isvalid'$ implies that necessarily $\iskstable = \iskstable'$ and $\isckstable = \isckstable'$.
\end{proof}

\begin{lemma}
  \label{lem:trans-1-progress}
  If $\conf \trans{1} \conf'$ and $\conf$ is consistent but not complete then either $\isvalid$ is a strict subset of $\isvalid'$ or $\ticks{\conf'}$ is non-empty.
\end{lemma}
\begin{proof}
  Notation-wise, assume $(G, k, \counter, \valuation, \result, \isvalid, \iskstable, \isckstable) = \conf$ and
  $(G, k, \counter', \valuation', \result, \isvalid', \iskstable', \isckstable) = \conf'$.
  We first make the following observation: for every $\alpha \in \rsub{\varphi} \setminus \isvalid$ there exists $\beta \in \rsub{\varphi} \setminus \isvalid$ such that $\sub{\beta} \subseteq \isvalid$.

  The proof of this observation is by induction on $\alpha$. For the base case when
  $\alpha = p, \neg p$, or $X$ we take $\beta = \alpha$. This suffices since
  $\alpha \not \in \isvalid$ and $\sub{\beta} = \emptyset \subseteq \isvalid$. For the inductive case when $\alpha$ is any other formula we check if there exists an immediate subformula $\psi \in
  \sub{\alpha} \setminus \isvalid$. If so, then we take the $\beta \in \rsub{\psi} \setminus \isvalid \subseteq \rsub{\alpha} \setminus \isvalid$ with $\sub{\beta} \subseteq \isvalid$ that we obtain by applying the induction hypothesis on $\psi$. If no such subformula $\psi$ exists, necessarily $\sub{\alpha} \subseteq \isvalid$, and it suffices to take $\beta = \alpha$.

  We now prove the lemma as follows. Since $\conf$ is not complete, $\varphi \not \in \isvalid$. By our observation, there is some $\beta \in \rsub{\varphi} \setminus \isvalid$ with $\sub{\beta} \subseteq \isvalid$. If $\beta \not \in \rfp{\varphi}$ then $\beta \in \isvalid'$ by definition of $\isvalid'$, hence $\isvalid \subset \isvalid'$. If $\beta \in \rfp{\varphi}$ then there are two possibilities.

  \begin{enumerate}
    \item $\counter(\beta) = k-1$. Then $\beta \in \isvalid'$ by definition of $\isvalid'$ and hence
      $\isvalid \subset \isvalid'$.
    \item $\counter(\beta) < k-1$. We claim that $\beta$ ticks in
      $\conf'$. Indeed, by Lemma~\ref{lem:trans-1-extends-valid},
      $\sub{\beta} \subseteq \isvalid \subseteq \isvalid'$. Furthermore,
      consistency of $\conf$ implies that $\isvalid$ is closed under subformulae,
      i.e., $\rsub{\isvalid} \subseteq \isvalid$. As such,
      $\tfp{\beta} \subseteq \isvalid$.  Hence, by the third property of
      consistency, every $\gamma \in \tfp{\beta}$ has
      $\counter(\gamma) = k-1$. \qedhere
  \end{enumerate}
\end{proof}

\begin{definition}
  We define a total order on $\sqsubseteq$ on counters: $\counter \sqsubseteq \counter'$ if for all $\alpha \in \rfp{\varphi}$ with $\counter(\alpha) > \counter'(\alpha)$ there exists $\beta \in \rfp{\varphi}$ with $\alpha \in \tfp{\varphi}$ such that $\counter(\beta) < \counter'(\beta)$. It is readily verified that $\sqsubseteq$ is indeed a total order. We write $\counter \sqsubset \counter'$ if $\counter \sqsubseteq \counter'$ and $\counter \not = \counter'$.
\end{definition}

\begin{definition}
  We define a strict partial order $\prec$ on configurations:
  $\conf \prec \conf'$ if
  \begin{enumerate}
    \item $k < k'$; or
    \item $k = k'$ and $\counter \sqsubset \counter'$; or
    \item $k = k'$ and $\counter = \counter'$ and $\isvalid \subset \isvalid'$.
  \end{enumerate}
\end{definition}

\begin{lemma}
  \label{lem:trans-1-2-progress}
  If $\conf \trans{1,2} \conf'$ and $\conf$ is consistent but not complete, then
  $\conf \prec \conf'$.
\end{lemma}
\begin{proof}
  Assume $\conf = \conf_1 \trans{1} \conf_2 \trans{2} \conf_3$.
  Notation-wise, assume
  $\conf_i = (k_i, \counter_i, \valuation_i, \result_i, \isvalid_i,
  \iskstable_i, \isckstable_i)$. Then $k_1 = k_2 = k_3$ and $\counter_1 = \counter_2$. There are
  two possibilities.
  \begin{enumerate}
    \item $\ticks{\conf_2} = \emptyset$. In this case, by definition of
      $\trans{2}$, $\counter_2 = \counter_3$ and $\isvalid_2 = \isvalid_3$.  By
      Lemma~\ref{lem:trans-1-progress},
      $\isvalid_1 \subset \isvalid_2 = \isvalid_3$. Hence $\conf \prec \conf'$.
    \item $\ticks{\conf_2} \not = \emptyset$. By definition of $\trans{1}$, $\counter_1 = \counter _2$. By definition of $\trans{2}$, $\counter_2 \sqsubset \counter_3$. This is because, for every fixpoint formula whose counter is set to $0$ in $\counter_3$, a ticking fixpoint ancestor has its counter incremented.   Hence $\conf \prec \conf'$. \qedhere
  \end{enumerate}
\end{proof}

             \end{toappendix}

Lemma's \ref{lem:trans-1-preservation-full}, \ref{lem:trans-2-preservation-full} and \ref{lem:trans-3-preservation} show preservation of coherence by $\trans{3,1,2}$. We can also show that the three transition types make \emph{progress} when executed in the order $\trans{3,1,2}$: $\trans{3}$ transitions to the next bound value when the input is complete  and is a no-op otherwise, while $\trans{1,2}$ change the configuration to become ``strictly more complete'', in the following sense.

            \begin{propositionrep}
              \label{prop:trans-1-2-reaches-complete-within-current-bound}
              Let $\conf$ be the initial configuration on $\varphi$ for $G$
              w.r.t. bound $k$. There exists $\conf'$ that is coherent and complete
              such that $\conf \trans{1,2}^* \conf'$.

\end{propositionrep}
            \begin{proof}
              Observe that $\trans{1}$ and $\trans{2}$ never change the bound.
              Therefore, for any $\conf''$ such that $\conf \trans{1,2} \conf''$, it follows that $\conf''$ is coherent (Lemma's \ref{lem:trans-1-preservation-full} and \ref{lem:trans-2-preservation-full}) and has the same bound $k$. By Lemma~\ref{lem:trans-1-2-progress}, $\conf \prec \conf''$. If $\conf''$ is complete then we are done. Otherwise, the result then follows by repeating this reasoning and observing that within the set of coherent configurations with bound value $k$, there is only a finite number of times that one can apply $\prec$ since  in every such configuration $\isvalid$ must be a subset of $\rsub{\varphi}$ and $\counter \leq k-1$.
            \end{proof}

            Combined with Proposition \ref{prop:smallest-stable-uniform-approximation-terminates-and-is-correct}, this yields correctness of the counting algorithm:
            \begin{proposition}
              \label{prop:incremental-correct}
              Let $\conf$ be the initial configuration $\conf$ on $\varphi$ for $G$ w.r.t. bound $1$. There exists a configuration $\conf'$ that is coherent, complete, and stable such that 
$\conf \trans{3,1,2}^*  \conf'$. \end{proposition}
            \begin{proof}
              Note that $\trans{3}$ is a no-op on configurations that are not complete. Since by Proposition~\ref{prop:trans-1-2-reaches-complete-within-current-bound} $\conf \trans{1,2}^* \conf''$ with $\conf''$ coherent, complete and having the same bound as $\conf$, there is also a sequence of transitions of the form $\conf \trans{3,1,2}^* \conf''$: before reaching completion we may vacuously introduce $\trans{3}$ before $\trans{1}$ since this is a no-op. Now two things may happen:
              \begin{itemize}
                \item $\conf''$ is stable, in which case  it suffices to take $\conf' = \conf''$;
                \item $\conf''$ is not stable. By executing $\trans{3}$ on $\conf''$ we then obtain the initial configuration $\conf_2$ that has bound $2$. By repeating our reasoning, but now starting from $\conf_2$ we know from Proposition~\ref{prop:smallest-stable-uniform-approximation-terminates-and-is-correct} that will eventually get the desired $\conf'$. \qedhere 
              \end{itemize}
            \end{proof}

\subsection{Implementing the counting algorithm in a halting-classifier recurrent \gnn}
\label{sec:implementing}

\begin{toappendix}
  \subsection{Proofs for Section~\ref{sec:implementing}}
\end{toappendix}

We next show how to encode configurations as labeled graphs and prove that that there exists a simple recurrent \gnns that simulates $\trans{3,2,1}^*$ on such encodings.

To define the encoding of a configuration $\conf$,
we first define \emph{local versions} of configurations. Intuitively, if $G$ is
the graph of $\conf$ then the local configuration of $\conf$ at $n \in G$ will
contain the information specific to $n$ stored in $\conf$, as well as the
information that is common to all nodes, such as $k$ and $\isvalid$.
Formally, for a function $M\colon A \to \pow{\nodes{G}}$ from some domain $A$ to
sets of nodes, we define the \emph{local version of $M$ at node $n$} to be the
  function $m\colon A \to \bools$ such that, for all $a \in A$, $m(a) = 1$ iff
  $n \in M(a)$.

\begin{definition}
  Let $\conf =(G, k, \counter, \valuation, \result, \isvalid, \iskstable, \isckstable)$ be a configuration and let $n$ be a node in $G$. 
  The \emph{local version} of $(\conf,n)$ is the tuple $(G(n), k, \counter, \lovaluation, \loresult, \isvalid, \loiskstable, \loisckstable)$ where
  \begin{compactitem}
  \item $G(n) \subseteq \allprops$ is the label of $n$ in $G$;
  \item $k \in \nats$ is the bound of $\conf$;
  \item $\isvalid$ is the validity set of $\conf$; and
  \item $\lovaluation$, $\loresult$, $\loiskstable$, and $\loisckstable$ are the local versions at $n$ of $\valuation, \result, \iskstable$, and $\isckstable$, respectively.
  \end{compactitem}
  The \emph{encoding of} $\conf$ is the graph $H$ that has the same nodes and
  edges as $G$, such that $H(n)$ is the local version of $(\conf,n)$, for
  every node $n \in G$.
\end{definition}

It is clear that when $A$ is a finite set, we may treat functions $A \to \bools$
and $A \to \nats$ as boolean resp. natural number vectors. Moreover, we may also
treat subsets of a finite universe $A$ as boolean vectors, since such subsets
are isomorphic to characteristic functions $A \to \bools$. Since all components
of local configurations are of this form, it follows that we may treat local
configurations as vectors over $\nats \cup \bools$, and hence also as vectors in
$\reals^d$ for some large enough value $d$. For example, for every
$p \in \allprops$ this vector has a component that is $1$ if $p \in G(n)$ and is
$0$ otherwise. In what follows, we hence treat local configurations as vectors
in $\reals^d$, with the understanding that its components carry a value from
either $\bools$ or $\nats$. To facilitate notation, we will refer to the
components of local configuration vector $\vec{x}$ using ``field access''
notation: e.g., $\vec{x}[\lovaluation(X)]$ is the boolean element of $\vec{x}$
that stores $\lovaluation(X)$. Specifically, $\vec{x}[\alpha \in \isvalid]$ is
the boolean element of $\vec{x}$ that is $1$ iff $\alpha \in
\isvalid$, and we similarly write $\vec{x}[p \in G(n)]$.

In the above sense, the encoding $H$ of $G$ is a $\reals^d$-labeled graph.

Let $f$ be a function mapping configurations
to configurations.
A label transformer $g\colon \graphs{\reals^d} \to \graphs{\reals^d}$ \emph{simulates} such a
function $f$ if for all configurations $\conf$, the equality $g(\enc(\conf)) = \enc(f(\conf))$ holds, where $\enc(\conf)$ denotes the encoding of $\conf$ as a labeled graph.

It is relatively straightforward to show that there is an \ac layer $L_i$
simulating $\trans{i}$, for every $1 \leq i \leq 3$. Hence, their sequential composition
$L_3; L_2; L_1$ simulates $\trans{3,2,1}$. It immediately follows that there
exists a ``multi-layer'' recurrent \gnn that iterates the composition $L_3; L_2; L_1$ to
simulate $\mu$-calculus formulae. This is not sufficient to prove
Theorem~\ref{thm:mu-calculus-expressible-in-simple-gnn}, however, since (i) we
have defined recurrent \gnns to iterate only a single \ac layer and (ii) this
layer must be simple for
Theorem~\ref{thm:mu-calculus-expressible-in-simple-gnn} to hold.

Unfortunately, we cannot simulate $\trans{2}$ and $\trans{3}$ by means of a
simple \ac layer: these transitions may cause  the counter $\counter(X)$
of a variable $X$ to be reset to zero if a certain boolean condition $b$ holds. To express
this by means of a \rfnn in the $\comb$ function of a \gnn, we must essentially
determine the new value of $\counter(X)$ by a function of the form ``if
$b = 1$ then $\counter(X)$ else $0$''. This function is non-continuous around
$b=1$, and therefore not expressible by a \rfnn, which always expresses a
continuous and piecewise-linear transformation. 

We hence need to work harder to obtain
Theorem~\ref{thm:mu-calculus-expressible-in-simple-gnn}. Our approach is conceptually simple: while we cannot
directly express ``if $b = 1$ then $\counter(X)$ else $0$'' in an \ac layer, we
may use the iteration capabilities of recurrent \gnns to repeatedly decrement $\counter(X)$ until it reaches zero. We have to take care, of course, that while we are doing this the state of the other configuration components is not incorrectly altered. 

Formally, an \emph{extended configuration} is a pair $(\conf,
\residual)$ with $\conf$ a configuration, and $\residual \subseteq \vars{\varphi}$. Intuitively, $\residual$ will hold the variables whose counter
value we need to keep decrementing. We also call $\residual$ the \emph{residual set}.

Given a configuration $\conf$, we define the \emph{partial} transition of type $2$ and type $3$, denoted $\partrans{2}$ and $\partrans{3}$ as follows.
The partial type-2 transition on $\conf$ yields the extended configuration $(\conf',\residual')$ where
\begin{inparaenum}[(i)]
\item $\conf'$ is defined such as in Definition~\ref{def:trans-2} except that $\counter'(X) \defeq \counter(X)$ for all $X$ with $\binds{\varphi}{X} \in \dep{\conf}$, and 
\item $\residual' \defeq \{ X \mid \binds{\varphi}{X} \in \dep{\conf} \}$.
\end{inparaenum}
The partial type-3 transition on $\conf$ yields the extended configuration $(\conf', \residual')$ where 
\begin{inparaenum}[(i)]
\item $\conf'$ is defined such as in Definition~\ref{def:trans-3} except that $\counter'(X) \defeq \counter(X)$ for all $X$, and
\item $\residual' \defeq \vars{\varphi}$.
\end{inparaenum}
These partial transitions hence delay setting variable counters to zero, but record in $\residual'$ for which variables this must still happen. For notational convenience, define $\conf \partrans{1} (\conf', \emptyset)$ whenever $\conf \trans{1} \conf'$.

For every $1 \leq i \leq 3$ we then define the \emph{extended type-$i$ transition} on extended configurations, denoted $(\conf, \residual) \etrans{i} (\conf', \residual')$. Here,  $(\conf', \residual') = (\conf, \residual)$ if $\residual \neq \emptyset$; otherwise, $(\conf', \residual')$ is the result of applying $\partrans{i}$ to $\conf$.

Finally, we define a \emph{reset transition} on extended configurations: on
$(\kappa, \residual)$ the resetting transition $\etrans{r}$ yields $(\kappa', \residual')$  where $\kappa'$ equals $\kappa$ on all components except $\counter$, and
\allowdisplaybreaks
\begin{align*}
  \counter'(X) & \defeq
    \begin{cases}
      \counter(X) -1 & \text{if } X \in \residual \text{ and } \counter(X) > 0 \\
      \counter(X) & \text{otherwise}
    \end{cases}\\
    \residual' & \defeq \{ X \in \residual \mid \counter(X) - 1 > 0\}
\end{align*}
Note in particular  that $(\conf', \residual') = (\conf, \residual)$ when 
$\residual = \emptyset$.

\begin{toappendix}
We write $\etrans{3,1,2,r^+}$ for the execution of $\etrans{3}$ followed by
$\etrans{1}$, then $\etrans{2}$, and then one or more executions of
$\etrans{r}$. $\etrans{3,1,2,r}$ is defined similarly, but ends with a single execution of $\etrans{r}$. We write $\etrans{r^+}$ for one ore more executions of $\etrans{r}$. 
  
\begin{lemma}
  \label{lem:etrans-3-1-2-r+-implies-etrans-3-or-312}
  Assume $(\conf, \emptyset) \etrans{3,1,2,r^+} (\conf', \emptyset)$.
  \begin{enumerate}
  \item If $\conf$ is complete, then $\conf \trans{3} \conf'$.
  \item If $\conf$ is not complete, then $\conf \trans{3,1,2} \conf'$.
  \end{enumerate}
\end{lemma}
\begin{proof}
  (1) If $\conf$ is complete, then
  $(\conf, \emptyset) \etrans{3} (\conf'', \residual)$ for some $\conf''$ and
  some non-empty $\residual$. Because $\residual$ is non-empty, $\etrans{1}$ and
  $\etrans{2}$ are the identity on $(\conf'', \residual)$. Hence, the execution of
  $\etrans{3,1,2,r^+}$ is equivalent to first executing $\etrans{3}$, and then
  $\etrans{r^+}$ until the residual set $\residual$ becomes empty.  Since
  $\etrans{3}$ executes $\trans{3}$ except for resetting the counter values, and
  $\etrans{r^+}$ only decrements the counters of variables in $\residual$ until zero
  , the net effect is the same as $\trans{3}$ on $\conf$. Hence
  $\conf \trans{3} \conf'$.

  (2) If $\conf$ is not complete, then $\etrans{3}$ is the identity on
  $(\conf, \emptyset)$. Therefore,
  $(\conf, \emptyset) \etrans{1,2,r^+} (\conf', \emptyset)$. When $\etrans{1}$
  executes on $(\conf, \emptyset)$ it returns $(\conf'', \emptyset)$ for some $ \conf''$ with
   $\conf \trans{1} \conf''$. On $(\conf'', \emptyset)$,
  $\etrans{2}$ returns $(\conf''', \residual)$ for some $\conf'''$ and $D$. We
  distinguish two cases.
  \begin{itemize}
  \item $\residual$ is empty. This only happens when
    $\dep{\conf''} = \emptyset$, and therefore, since we defined $\partrans{2}$
    to equal $\trans{2}$ except for the counter of variables in
    $\{ X \mid \binds{\varphi}{X} \in \dep{\conf''}\}$, we see that
    $\conf''\trans{2} \conf'''$. Executing $\etrans{r^+}$ on
    $(\conf''', \emptyset)$ is the identity, and hence $\conf''' =
    \conf'$. We may hence conclude that
    $\conf \trans{3} \conf \trans{1} \conf'' \trans{2} \conf'$, as desired.

  \item Otherwise, $\residual$ is non-empty. In this case, $\etrans{2}$ executed
    the same logic as $\trans{2}$ except for the resetting the counter values of
    variables in $\{X \mid \binds{\varphi}{X} \in \dep{\conf''} \}$. The
    subsequent execution of $\etrans{r^+}$ decrements the counters of those
    variables until zero and leaves everything else untouched. The net effect is
    the same as $\trans{2}$ on $\conf''$. Therefore,
    $\conf \trans{3} \conf \trans{1} \conf'' \trans{2} \conf'$, as desired. \qedhere
  \end{itemize}
\end{proof}

\begin{lemma}
  \label{lem:trans-3-1-2-implies-one-or-two-etrans-3-1-2-r+}
  Assume $\conf \trans{3,1,2} \conf'$.
  \begin{enumerate}
  \item If $\conf$ is not complete, then $(\conf, \emptyset) \etrans{3,1,2,r^+} (\conf', \emptyset)$.
  \item If $\conf$ is complete, then there exists for $\conf''$ such that
$(\conf, \emptyset) \etrans{3,1,2,r^+} (\conf'', \emptyset) \etrans{3,1,2,r^+} (\conf', \emptyset)$ 
  \end{enumerate}
\end{lemma}
\begin{proof}
  (1) If $\conf$ is not complete, then
  $\conf \trans{3} \conf \trans{1} \conf'' \trans{2} \conf'$ for some configuration $\conf''$. By definition of $\etrans{3}$ and $\etrans{1}$,
  \[ (\conf,\emptyset) \etrans{3} (\conf,\emptyset) \etrans{1} (\conf'',\emptyset). \]

  Let $(\conf''', \residual)$ be the result of executing $\etrans{2}$ on
  $(\conf'',\emptyset)$. By definition of $\etrans{2}$, $\conf'''$ equals
  $\conf'$ except that possibly the counter values of some variables in $\conf'$
  are not reset to zero, and $\residual$ stores the variables for which this is
  the case. If $\residual \not = \emptyset$ then We can make the counter value of these variables zero, and leave
  everything else untouched, by repeatedly executing $\etrans{r}$ until $\residual$ becomes $\emptyset$. If $\residual = \emptyset$, then $\conf'''$ itself is already the result of applying $\trans{2}$ on $\conf''$, i.e., $\conf''' = \conf'$. Also observe that in this case, by definition, $\etrans{r^+}$ on $(\conf''', \residual)$ yields $(\conf''', \residual)$ itself. Hence, in conclusion, both in the case where $\residual = \emptyset$ and $\residual \not = \emptyset$ we have:
  \[ (\conf,\emptyset) \etrans{3} (\conf,\emptyset)  \etrans{1} (\conf'',\emptyset) \etrans{2} (\conf''',\residual) \etrans{r^+} (\conf',\emptyset). \]

  (2) If $\conf$ is complete, then $\trans{3}$ returns the initial configuration
  $\conf''$ for $k+1$, where $k$ is the bound in $\conf$, and
  $\conf''\trans{1,2} \conf'$. By definition of $\etrans{3}$,
  $(\conf,\emptyset) \etrans{3} (\conf''',\residual)$ with $\residual$ non-empty
  and $\conf'''$ equal to $\conf''$, except that no variable has its counter
  value reset. Since $\residual$ is non-empty $\etrans{1,2}$ is the identity
  when executed on $(\conf''', \residual)$. A subsequent execution of
  $\etrans{r^+}$ on $(\conf''', \residual)$ then takes care of the resetting of
  the counter values.  As such
  $(\conf, \emptyset) \etrans{3,1,2,r^+} (\conf'', \emptyset)$. Because
  $\conf''$ is an initial configuration, it is not complete. Therefore, applying
  $\trans{3}$ again to $\conf''$ simply yields $\conf''$. Hence,
  $\conf'' \trans{3} \conf'' \trans{1,2} \conf'$, i.e.,
  $\conf'' \trans{3,1,2} \conf'$. By applying point (1) to
  $\conf'' \trans{3,1,2} \conf'$ it follows that
  $(\conf'',\emptyset) \etrans{3,1,2,r^+} (\conf',\emptyset)$. Therefore,
  $(\conf',\emptyset) \etrans{3,1,2,r^+} (\conf'',\emptyset) \etrans{3,1,2,r^+}
  (\conf',\emptyset)$.
\end{proof}

\begin{lemma}
  \label{lemma:iterated-etrans-3-1-2-r+-same-as-iterated-etrans-3-1-2-r}
  Let $\conf$ and $\conf'$ be configurations. Then  $(\conf, \emptyset) \etrans{3,1,2,r^+}^* (\conf', \emptyset)$ if, and only if,
  $(\conf, \emptyset) \etrans{3,1,2,r}^* (\conf', \emptyset)$.
\end{lemma}
\begin{proof}
  ($\Rightarrow$) Assume
  $(\conf, \emptyset) \etrans{3,1,2,r^+}^* (\conf, \emptyset)$. The proof is by
  induction on the number of times that $\etrans{3,1,2,r^+}$ is executed to
  obtain $(\conf', \emptyset)$. The base case, with zero executions, is trivial.  Otherwise,
  \[ (\conf,\emptyset) \etrans{3,1,2,r^+} (\conf''',\emptyset)  \etrans{3,1,2,r^+}^* (\conf', \emptyset)\] for some $\conf'''$. We can write this as 
  \[ (\conf,\emptyset) \etrans{3,1,2,r} (\conf'',\residual) \etrans{r}^*
    (\conf''',\emptyset) \etrans{3,1,2,r^+}^* \conf'\] for some
  $(\conf'',\residual)$, where the $\etrans{r}^*$ is executed until $\residual$ becomes empty. Note that on inputs where $\residual$ is non-empty, $\etrans{r}^*$ is equivalent to $\etrans{3,1,2,r}^*$. Therefore, 
  \[ (\conf,\emptyset) \etrans{3,1,2,r} (\conf'',\residual) \etrans{3,1,2,r}^*
    (\conf''',\emptyset) \etrans{3,1,2,r^+}^* \conf'\] and hence, by induction hypothesis, also 
  \[ (\conf,\emptyset) \etrans{3,1,2,r} (\conf'',\residual) \etrans{3,1,2,r}^*
    (\conf''',\emptyset) \etrans{3,1,2,r}^* \conf',\]
  as desired.

  ($\Leftarrow$) Assume $(\conf, \emptyset) \etrans{3,1,2,r}^* (\conf', \emptyset)$. The proof is by
  induction on the number of times that $\etrans{3,1,2,r}$ is executed to
  obtain $(\conf', \emptyset)$. The base case, with zero executions, is trivial.  Otherwise,  $(\conf,\emptyset) \etrans{3,1,2,r} (\conf_0,\residual_0)  \etrans{3,1,2,r}^* (\conf',\emptyset)$ for some $\conf_0$ and $\residual_0$.  We distinguish two cases.
  \begin{itemize}
  \item $\residual_0 = \emptyset$, then
    $(\conf,\emptyset) \etrans{3,1,2,r} (\conf_0,\emptyset)$ hence also
    $(\conf,\emptyset) \etrans{3,1,2,r^+} (\conf_0,\emptyset)$. The result then
    follows from the induction hypothesis.
  \item $\residual_0 \not = \emptyset$. Then we must execute $\etrans{3,1,2,r}$
    $\ell \geq 1$ times before $\residual_0$ becomes empty. I.e., the derivation
    of $(\conf'',\residual_0) \etrans{3,1,2,r}^* (\conf',\emptyset)$ is of the
    form
    \begin{multline*}
(\conf'',\residual_0)  \etrans{3,1,2,r}  (\conf_1,\residual_1)   \etrans{3,1,2,r} \dots \\ \etrans{3,1,2,r} (\conf_{\ell-1},\residual_{\ell-1})  \etrans{3,1,2,r} (\conf_{\ell}, \emptyset) \\ \etrans{3,1,2,r}^* (\conf', \emptyset)      
    \end{multline*}
    It is straightforward to observe that on inputs where $\residual \not =\emptyset$, $\etrans{3,1,2,r}$ is equivalent to $\etrans{r}$. Therefore, 
    \begin{multline*}
(\conf'',\residual_0)  \etrans{r}  (\conf_1,\residual_1)   \etrans{r} \dots \\ \etrans{r} (\conf_{\ell-1},\residual_{\ell-1})  \etrans{r} (\conf_{\ell}, \emptyset) \\ \etrans{3,1,2,r}^* (\conf', \emptyset)      
    \end{multline*}
    I.e., $(\conf'',\residual_0) \etrans{r^+} (\conf_\ell,
    \emptyset)$. Therefore,
    \[ (\conf, \emptyset) \etrans{3,1,2,r^+} (\conf_\ell,\emptyset)
      \etrans{3,1,2,r}^* (\conf',\emptyset),\] from which the result follows by
    the induction hypothesis applied to
    $(\conf_\ell,\emptyset) \etrans{3,1,2,r}^* (\conf',\emptyset)$. \qedhere
  \end{itemize}
\end{proof}

 \end{toappendix}

\begin{propositionrep}
  \label{prop:extended-trans-equal-normal-trans}
  Let $\conf, \conf'$ be configurations with  $\conf'$ complete.
  Then $\conf \trans{3,1,2}^* \conf'$ if, and only if,
  $(\conf, \emptyset) \etrans{3,1,2,r}^* (\conf', \emptyset)$.
\end{propositionrep}
\begin{proofsketch}
  \vspace{-1ex}
  We only illustrate the $\Rightarrow$ direction, the converse direction proceeds similarly but additionally exploits the completeness of $\conf'$. 
  If $\conf$ is not complete, then $\trans{3}$ and $\etrans{3}$
  are the identity on $\conf$ resp. $(\conf, \emptyset)$. As such, it is not
  difficult to see that if $\conf$ is not complete we may mimic
  $\trans{3,1,2}$ by executing $\etrans{3, 1,2}$ on $(\conf, \emptyset)$ followed by zero or more executions of $\etrans{r}$ until the residual set
  becomes empty. Since each extended transition $\etrans{i}$ with
  $1 \leq i \leq 3$ acts as the identity on extended configurations for which
  the residual set is no-empty, we may also execute $\etrans{r}$ on extended
  configurations with non-empty residual set by means of
  $\etrans{3,1,2,r}$. Consequently, on incomplete configurations we may mimic
  $\trans{3,1,2}$ by executing $\etrans{3,1,2,r}^*$.
  
  If $\conf$ is complete, then $\trans{3}$ yields a non-complete configuration,
  say $\conf''$.  Using analogous reasoning as before, we can argue that we may
  mimic $\trans{3}$ on $\conf$ by means of $\etrans{3,1,2,r}^*$.  Because
  $\conf''$ is not-complete, the subsequent execution of $\trans{1,2}$ on
  $\conf''$ is equivalent to execution of $\trans{3,1,2}$ on $\conf''$. The
  latter can be mimicked by means of $\etrans{3,1,2,r}^*$ by our reasoning in
  the previous case. Consequently, $\trans{3,1,2}$ is mimicked by means of two
  applications of $\etrans{3,1,2,r}^*$. \qedhere

\end{proofsketch}
\begin{proof}
  ($\Rightarrow$).  Assume that $\conf \trans{3,1,2}^* \conf'$. We need to show
  that $(\conf, \emptyset) \etrans{3,1,2,r}^* (\conf', \emptyset)$. The proof is
  by induction on number of times that $\trans{3,1,2}$ is executed to obtain
  $\conf'$. The base case, with zero executions, is trivial. Otherwise,
  $\conf \trans{3,1,2} \conf'' \trans{3,1,2}^* \conf'$ for some $\conf'' $.  By
  Lemma~\ref{lem:trans-3-1-2-implies-one-or-two-etrans-3-1-2-r+} and the
  induction hypothesis, also
  \[ (\conf,\emptyset) \etrans{3,1,2,r^+}^* (\conf'', \emptyset)
    \etrans{3,1,2,r}^* (\conf',\emptyset)\] Hence, by
  Lemma~\ref{lemma:iterated-etrans-3-1-2-r+-same-as-iterated-etrans-3-1-2-r}, also   \[ (\conf,\emptyset) \etrans{3,1,2,r}^* (\conf'', \emptyset)
    \etrans{3,1,2,r}^* (\conf', \emptyset).\]

  ($\Leftarrow$). Assume
  $(\conf,\emptyset) \etrans{3,1,2,r}^* (\conf'', \emptyset)$. By Lemma~
  \ref{lemma:iterated-etrans-3-1-2-r+-same-as-iterated-etrans-3-1-2-r}, also
  $(\conf,\emptyset) \etrans{3,1,2,r^+}^* (\conf'', \emptyset)$. We show that
  this implies $\conf \trans{3,1,2}^* \conf'$.  The proof is by induction on the
  number of times $\etrans{3,1,2,r^+}$ is executed.  The base case, with zero
  executions, is trivial. Otherwise,
  \[ (\conf,\emptyset) \etrans{3,1,2,r^+} (\conf'', \emptyset)
    \etrans{3,1,2,r^+}^* (\conf', \emptyset).\]
 We distinguish two cases.
  \begin{itemize}
  \item $\conf$ itself is complete. By
    Lemma~\ref{lem:etrans-3-1-2-r+-implies-etrans-3-or-312}(1), $\conf \trans{3} \conf''$, which implies that $\conf''$ is not complete. As such, $\conf' \not = \conf''$, and thus there exists some $\conf'''$ such that 
  \[ (\conf'',\emptyset) \etrans{3,1,2,r^+} (\conf''', \emptyset)
    \etrans{3,1,2,r^+}^* (\conf', \emptyset).\]
  By Lemma~\ref{lem:etrans-3-1-2-r+-implies-etrans-3-or-312}(2), $\conf'' \trans{3,1,2} \conf'''$. As such,
  \[ \conf \trans{3} \conf'' \trans{3,1,2} \conf'''\] From this and the
  observation that $\trans{3}$ is the identity on non-complete configurations,
  and $\conf''$ is not complete, we may therefore conclude that
  $\conf \trans{3,1,2} \conf'''$. Furthermore, since
  $(\conf''', \emptyset) \etrans{3,1,2,r^+}^* (\conf', \emptyset)$, also
  $\conf''' \trans{3,1,2}^* \conf'$ by induction hypothesis. Therefore,
  $\conf \trans{3,1,2}^* \conf''$, as desired.
\item $\conf$ is not complete. Then $\conf \trans{3,1,2} \conf''$ by Lemma~\ref{lem:etrans-3-1-2-r+-implies-etrans-3-or-312}(2)  and $\conf'' \trans{3,1,2}^* \conf'$ by induction hypothesis. Hence $\conf \trans{3,1,2}^* \conf''$. \qedhere
  \end{itemize}
\end{proof}

We can encode extended configurations as labeled graphs similarly to how we
encode configurations as labeled graphs: in the local configuration of
$(\conf, \residual)$ at node $n$ we now also include $\residual$ at every
node. The concept of a label transformer simulating a function on extended
configurations is defined in the obvious way.

\begin{propositionrep}
  \label{prop:simulate-single-extended-trans}
  There exists a simple \ac layer simulating $\etrans{3,1}$, as well as
  \rfnns whose lifting simulate $\etrans{2}$ and
  $\etrans{r}$.
\end{propositionrep}
\begin{proofsketch}
  The crux is that local versions of extended configurations are vectors whose  elements are all in $\bools$ or $\nats$. It is well-known that, on such vectors, \rfnns can express all functions defined by boolean combinations of (i) the $\bools$ input elements and (ii) comparisons on the $\nats$ elements. For instance, if $a,b\in \bools$ and $c \in \nats$ then then function $\phi(a,b,c) = \neg(a \wedge \neg b) \vee (c > 1)$ is definable by an \rfnn. \inFullVersion{See the appendix for an illustration.}

We will only need such functions to define output local configuration vectors of $\etrans{3,1}$, $\etrans{2}$, and $\etrans{r}$. For $\etrans{3,1}$ we also need the message passing capability of \ac layers.

Let us illustrate how to simulate $\etrans{3,1}$, focusing on a single output
element. Assume that $(\conf, \residual) \etrans{3,1} (\conf', \residual')$. Let
$H$ and $H'$ be the encodings of $(\conf, \residual)$ and
$(\conf', \residual')$, respectively. Let $n \in \nodes{H} = \nodes{H'}$. Then
$H(n)$ is the input local configuration vector at $n$, and $H'(n)$ the output
vector. We illustrate only how to define the output element
$H'(n)[\loresult(\atleast{\ell}\, \psi)]$ with $\atleast{\ell}\, \psi$ a
subformula of $\varphi$.\footnote{Recall that
  $H'(n)[\loresult(\atleast{\ell}\, \psi)]$ stores whether
  $n \in \result'(\atleast{\ell}\, \psi)$ with $\result'$ the result assignment
  of $\conf'$.}  According to the definition of $\etrans{3,1}$:
\begin{itemize}
\item if $\residual' \not = \emptyset$ then $H'(n)[\loresult(\atleast{\ell}\, \psi] = H(n)[\loresult(\atleast{\ell}\, \psi]$;
\item if $\residual'  = \emptyset$ and $\conf$ is complete then $H'(n)[\loresult(\atleast{\ell}\, \psi] = 0$ since $\trans{3}$ moves to the next initial $k$-configuraton;
\item otherwise, $\residual' = \emptyset$ and $\conf$ is not complete, and $\etrans{3}$ is the identity on $(\conf,\residual)$ and hence $(\conf', \residual')$ is the result of applying $\etrans{1}$ on $(\conf, \residual)$. Therefore, according to Definition~\ref{def:trans-one}, in this case, $H'(n)[\loresult(\atleast{\ell}\, \psi)] = 1$ if, and only if, $\left(\sum_{m \in \out{H}{n}} H(m)[\loresult(\psi)]\right) \geq \ell$. Note that in an \ac layer, $\sum_{m \in \out{H}{n}} H(m)[\loresult(\psi)]$ is provided by the $\agg$ function, which aggregates the local vectors of neighboring nodes, so this remains a comparison of an input feature vector element.
\end{itemize}

In each of these three cases, $H'(n)[\loresult(\atleast{\ell}\, \psi]$ is hence determined by a boolean combination of input boolean elements and natural number comparisons. The three conditions themselves can also be expressed as boolean combinations: $\residual'= \emptyset$ is equivalent to $\bigvee_{X \in \vars{\varphi}} H(n)[X \in \residual]$ while completeness of $\conf$ is equivalent to $H(n)[\varphi \in \isvalid]$. Therefore, the entire computation of $H'(n)[\loresult(\atleast{\ell}\, \psi]$ is definable by a simple \ac layer.
\end{proofsketch}
\begin{proof}
  The full proof follows the reasoning of the proof sketch. We here only
  illustrate two things that are worth mentioning explicitly.

  (1) First, the simulation of $\etrans{2}$ hinges on the fact that for any variable $X$ we can express the condition $\binds{\varphi}{X} \in \ticks{\conf}$ as well as $\binds{\varphi}{X} \in \dep{\conf}$ as boolean formulae $\tau_X$ resp $\delta_X$  on local configuration vectors. Suppose that $\vec{x}$ is a local configuration vector at some node of a graph encoding $(\kappa,\residual)$. Then $\tau_X$ follows the definition of $\binds{\varphi}{X}$ ticking:
  \begin{align*}
    \tau_X & = \bigwedge_{\beta \in \sub{\binds{\varphi}{X}}} \vec{x}[F(\beta)] \\  & \quad \wedge \vec{x}[\counter(\binds{\varphi}{X})] + 1 > \vec{x}[k]\\
&  \quad \wedge \bigwedge_{\beta \in  \sub{\binds{\varphi}{X}}} \vec{x}[\counter(\beta)] + 1 = k
  \end{align*}
  
  The formula $\delta_X$ needs a little more work.  For any variable $X$, define, $\delta_X^0$ to be true if and only if $\binds{\varphi}{X}$ has some free variable that ticks in $\conf$, i.e.
  \[ \delta_X^0 = \bigvee_{Z \in \free{\binds{\varphi}{X}}} \tau_Z. \]
  For $i>0$, define $\delta_X^i$ to be true if $\delta_X^{i-1}$ is true or if $\binds{\varphi}{X}$  has a free variable for which that is in $\delta_Z^{i-1}$ is true.
  \[ \delta_X^i = \delta_X^{i-1} \vee \bigvee_{Z \in \free{\binds{\varphi}{X}}}
    \delta_Z^{i-1} \] Let $q$ be the nesting depth fixpoint formulae in
  $\varphi$. Then it is not difficult to see that $X \in \dep{\conf}$ if, and
  only if, $\delta_X^q$ is true.
  
  The simulation of $\etrans{2}$ is then rather straightforward. For example, to
  determine the output counter value of
  $\alpha = \binds{\varphi}{X} \in \rfp{\varphi}$, we reason as follows. If the
  residual set $\residual$ is non-empty, then the counter must not change. If
  the residual set is empty then the counter must increase by one if
  $\tau_X = 1$. Observe that when simulated by a \rfnn, the boolean formula
  $(\tau_X \wedge \residual = \emptyset)$ yields $1$ exactly when the counter
  needs to increase by one.\footnote{Recall that in the proof sketch we have
    already argued that $\residual = \emptyset$ is expressible as a boolean
    formula.}  Therefore we may compute $\counter'(\alpha)$ by computing
  $\vec{x}[\counter(\alpha)] + (\tau_X \wedge \residual = \emptyset)$. The other
  elements follow a similar reasoning. In particular, we never need the
  aggregate features computed by $\agg$, which is why $\etrans{2}$ is simulated
  simply by lifting an \rfnn to operate on labeled graphs.

  (2) For simulating $\etrans{r}$ we can follow a similar reasoning. For all $X$
  \begin{align*}
    \counter'(X)  & = \counter(X) - \vec{x}[X \in \residual] \\
  \end{align*}
  and in the output we have $[X \in \residual'] = \vec{x}[X \in \residual] \wedge (\vec{x}[\counter(X)] > 2)$. All other vector elements are left unchanged.
\end{proof}

\begin{corollary}
  \label{cor:simulate-extended-trans-312r}
  There exists a simple \ac layer simulating $\etrans{3,1,2,r}$.
\end{corollary}
\begin{proof}
  Observe that simple \ac layers are closed under composition with lifted
  \rfnns: if $L\colon \graphs{\reals^d} \to \graphs{\reals^d}$ is a simple
  \ac layer  and $f\colon \reals^d \to \reals^d$ is
  a \rfnn then $\lift{f} \circ L$ is also expressible by a simple \ac
  layer, obtained by taking the \rfnn $g\colon \reals^d \to \reals^d$ in
  $L$'s $\comb$ function, and replacing it by $f \circ g$. \qedhere
\end{proof}

We now have all the ingredients to prove
Theorem~\ref{thm:mu-calculus-expressible-in-simple-gnn}.

\begin{proof}[Proof of Theorem~\ref{thm:mu-calculus-expressible-in-simple-gnn}]
  Let $\varphi$ be a \mml sentence and $G$ the input to $\varphi$.  Let $d$ be
  the dimension necessary to encode local versions of extended configurations as
  vectors in $\reals^d$.

  Fix $\netw = (\init, L, \hlt, \readout)$ be the recurrent \gnn over
  $\pow{\allprops}$ of dimension $d$ where
  \begin{itemize}
  \item $\init\colon \pow{\allprops} \to \reals^d$ maps each finite set $P$ of
    proposition symbols to the initial local extended configuration of $\varphi$
    w.r.t. $k=1$, which is $(P, 1, v, \emptyset, r, s, t, \emptyset)$ with $v,s,t$
    mapping all formulae to $0$, and $t$ mapping all formulae to $1$. It is
    straighforward to see that $\lift{\init}$, when executed on
    $G$, returns the encoding of
    $(\conf, \emptyset)$ with $\conf$ the initial configuration of $\varphi$ on
    $G$ w.r.t. $k =1$.
  \item $L$ is the simple \ac layer simulating $\etrans{3,1,2,r}$, which exists by  Corollary~\ref{cor:simulate-extended-trans-312r}.
  \item $\hlt\colon \reals^d \to \bools$ is the \rfnn that, given the local
    version $\vec{x_n}$ of an extended configuration $(\conf, \residual)$ at
    node $n$ returns $1$ if and only if $\conf$ is complete (i.e.,
    $\vec{x_n}[\varphi \in \isvalid] = 1$), $\varphi$ is $k$-stable at $(G,\valuation,n)$,
    (i.e., $\vec{x_n}[\loiskstable(\varphi)] = 1$), and $\residual$ is empty (i.e., $\neg(\vee_{X \in \vars{\varphi}} \vec{x_n}[X \in \residual]) = 1$). On a graph $H$ that encodes an extended configuration $(\conf,\residual)$, $\lift{\hlt}(H)$ has all nodes labeled $1$ if, and only if, $\conf$ is  complete and stable, and $\residual = \emptyset$.
  \item $\readout\colon \reals^d \to \bools$ is the function that on $\vec{x_n}$
    which is the encoding of $(\conf, \residual)$ at node $n$, returns the
    element $\vec{x}[\loresult(\varphi)]$ of $\vec{x}$ that stores whether
    $n \in \result(\varphi)$.
  \end{itemize}

$\netw$ expresses the node classifier defined by $\varphi$ by the combination
  of   Proposition~\ref{prop:smallest-stable-uniform-approximation-terminates-and-is-correct}, ~\ref{prop:incremental-correct}, ~\ref{prop:extended-trans-equal-normal-trans}, and Corollary ~\ref{cor:simulate-extended-trans-312r}. Strictly speaking, $\netw$ is not simple since $\hlt$ need not be simple. This is easily solved, however, by extending local configurations with an extra boolean element that is used to store the result of $\hlt$, and modifying $L$ so that it also computes $\hlt$ and stores it in this extra element. Then, $\hlt$ can just read this element instead.
\end{proof}

\begin{toappendix}
    \subsection{On the expressive power of \rfnns.}
\label{sec:expressive-power-rfnns}

In the proof of Proposition~\ref{prop:simulate-single-extended-trans} and Theorem~\ref{thm:mu-calculus-expressible-in-simple-gnn} we have used the fact that, given input vectors where all elements are in $\bools$ or $\nats$, \rfnns can express all functions defined by boolean combinations of (i) the boolean input elements and (ii) comparisons on the $\nats$ elements and additions on the $\nats$ elements.
We illustrate this here in some more detail.

Recall that \(\relu(x)=\max(0,x)\) for $x \in \reals$.  Define
\begin{align}
  \operatorname{clip}(x)=\relu(\relu(x)-\relu(x-1)),
\end{align}
Clearly, $\operatorname{clip}$ maps integer inputs exactly onto Boolean values (0 if $x$ is an integer with $x \leq 0$, and 1 otherwise). We interpret logical truth as \(1\) and logical falsehood as \(0\). Using this interpretation, we define the strict greater-than relation and Boolean conjunction as:
\begin{align}
  (a>b)=\operatorname{clip}(a-b),\quad\quad(a\land b)=\operatorname{clip}(a+b-1).
\end{align}
It is readily verified that the first is correct for natural numbers $a,b \in \nats$, and the second is correct for booleans $a,b \in \bools$.
Furthermore, boolean negation is defined as 
\begin{align}
  (\neg a)=1 - a
\end{align}

From this, and the well-known fact that  \rfnns are closed under sequential composition, parallel composition, and concatenation, our claim follows. Here, the sequential composition of $f \colon \reals^p \to \reals^q$  with $g\colon \reals^q \to \reals^r$ is the function $g \circ f\colon \reals^p \to \reals^r$ such that $(g \circ f)(\vec{x}) = g(f(\vec{x}))$ for all $\vec{x} \in \reals^p$. The parallel composition of $f \colon \reals^p \to \reals^q$ with $f' \colon \reals^{p'} \to \reals^{q'}$ is the function $f \| f'\colon \reals^{p + p'} \to \reals^{q + q'}$ such that $(f \| f')(\vec{x} \concat \vec{x'}) = f(\vec{x}) \concat f'(\vec{x'})$ for all $\vec{x} \in \reals^{p}$ and $\vec{x'} \in \reals^{p'}$.  (Recall that $\vec{x} \concat \vec{x'}$ denotes vector concatenation.) Finally, the concatenation of $f\colon \colon \reals^p \to \reals^q$ with $h \colon \reals^p \to \reals^{q'}$ is the function $f \concat h\colon \reals^p \to \reals^{q+q'}$ such that $(f \concat h)(\vec{x}) = f(\vec{x}) \concat h(\vec{x})$ for all $\vec{x} \in \reals^p$.

 \end{toappendix}

 \section{Concluding remarks}
\label{sec:conclusion}

We have shown that ``bare bones'' recurrent GNNs, without any
features that allow to determine the graph size, can express
$\mu$-calculus, under a terminating semantics.
Specifically, the GNN can be halted as soon as we see a stopping
bit of one in every node. Morever, it is guaranteed that this
will happen after a finite number of iterations (polynomial in
the graph size).  Proving this possibility result turned out to
be surprisingly subtle and intricate.

An interesting question for further research is whether it is
possible to make the termination fully convergent, as intended in
the original recurrent GNN proposal by Scarselli et
al.~\shortcite{original-gnn}.  Specifically, is it possible to
express every $\mml$ formula by a recurrent GNN, with stopping
bits and our halting guarantee, and moreover such that the entire
feature vector of all nodes will stabilize?

It is also natural to wonder about the converse direction: do
recurrent GNN classifiers always fall within $\mml$?  The
unreserved statement certainly does not hold, even for
fixed-depth GNNs \cite{barcelo-log-expr-gnn}. As already mentioned, Pflueger et
al.~obtained a converse result relative to node classifiers
expressible in ``local'' monadic fixpoint logic.

What if we relativize more generally to MSO (monadic second-order
logic) node classifiers and strongly connected graphs?
In restriction to strongly connected graphs, g-bisimilarity and total surjective g-bisimilarity coincide, and hence here
recurrent GNNs are invariant under graded bisimulations
(Section~\ref{sec:recurrent-gnns}).
Since unary MSO formulas
invariant under graded bisimilarity are expressible in $\mml$
\cite{DBLP:journals/tcs/Walukiewicz02,DBLP:conf/lics/JaninL01},
it would then immediately follow that recurrent GNN node
classifiers in MSO fall within $\mml$ (when restricted to strongly connected graphs), were it not for the caveat
that the cited Janin--Walukiewicz theorem (as well as its graded
version) assumes invariance of the MSO formula over all graphs,
finite or infinite.  In contrast, GNNs are designed to operate on
finite graphs only.

There may be a way out of this conundrum, as a breakthrough,
solving the open problem of proving the finitary version of the
cited Janin--Walukiewicz theorem, has recently been announced
\cite{colc-msomufinite}.  Another way to sidestep the situation
is to agree on a reasonable definition of GNNs working on
infinite graphs.  This is an interesting line of research and
promising work already exists for graphons \cite{boeker-fine-graphons}.
Of course, the requirement that the recurrent GNN node classifier is
expressible in MSO becomes stronger and possibly unnatural in
this manner, since MSO cannot express finiteness. 
The general case where graphs are not necessarily strongly connected
remains wide open.

\section*{Acknowledgments}
This work was supported by the Bijzonder Onderzoeksfonds (BOF) of Hasselt University  Grant No. BOF20ZAP02; by the Research Foundation Flanders (FWO) under research project Grant No. G019222N; and by the Flanders AI (FAIR) research program.

\bibliographystyle{kr}
\bibliography{references,database}

\newpage

\end{document}